\documentclass[11pt]{article}

\usepackage[
  shownumpages,               
  bgcolor={245,245,250},      
  braincolor={blue},          
  citingstyle=authoryear,     
  bibliostyle=plainnat,       
  bibfile=refs                
]{brainlab}

\usepackage{hyperref}       
\usepackage{url}            
\usepackage{booktabs}       
\usepackage{amsfonts}       
\usepackage{microtype}      
\usepackage{xcolor}         
\usepackage{mathtools}      

\newtheorem{lemma}[theorem]{Lemma}

\usepackage{multirow}
\usepackage{graphicx}
\usepackage{subcaption}

\usepackage{wrapfig}

\usepackage[textsize=tiny]{todonotes}

\setbrainmeta{
  title={Where Does Warm-Up Come From? Adaptive Scheduling for Norm-Constrained Optimizers \vspace{-3mm}},
  authors={
    Artem Riabinin\textsuperscript{1}, 
    Andrey Veprikov\textsuperscript{1, 2}, 
    Arman Bolatov\textsuperscript{3}, 
    Martin Takáč\textsuperscript{3}, 
    Aleksandr Beznosikov\textsuperscript{1, 2, 4}
  },
  affiliations={
    \textsuperscript{1}Basic Research of Artificial Intelligence Laboratory (BRAIn Lab)\\
    \textsuperscript{2}Federated Learning Problems Laboratory\\
    \textsuperscript{3}Mohamed bin Zayed University of Artificial Intelligence (MBZUAI)\\
    \textsuperscript{4}Innopolis University\\
  },
  abstract={
    \vspace{-5mm}
    We study adaptive learning rate scheduling for norm-constrained optimizers (e.g., Muon and Lion). 
    We introduce a generalized smoothness assumption under which local curvature decreases with the suboptimality gap and empirically verify that this behavior holds along optimization trajectories. 
    Under this assumption, we establish convergence guarantees under an appropriate choice of learning rate, for which warm-up followed by decay arises naturally from the proof rather than being imposed heuristically.\\
    Building on this theory, we develop a practical learning rate scheduler that relies only on standard hyperparameters and adapts the warm-up duration automatically at the beginning of training.
    We evaluate this method on large language model pretraining with LLaMA architectures and show that our adaptive warm-up selection consistently outperforms or at least matches the best manually tuned warm-up schedules across all considered setups, without additional hyperparameter search. Our source code is available at \url{https://github.com/brain-lab-research/llm-baselines/tree/warmup}.
  },
}

\begin{document}
\begin{mainpart}
\vspace{-10mm}
\section{Introduction}

In this paper, we consider the problem of training large language models (LLMs), which can be formulated as the following optimization problem:
\begin{equation}
\label{main_problem}
f^\star := \min_{x \in \mathcal{X}} f(x),
\end{equation}
where $f : \mathcal{X} \to \mathbb{R}$ denotes the loss of the model $x$ with a parameter space $\mathcal{X} := \{(W_1, \ldots, W_L) \mid W_i \in \mathbb{R}^{m_i \times n_i}\}$ representing the collection of $L$ model's layers.

Nowadays, classical method is to solve~\eqref{main_problem} using norm-constrained optimizers, where the update direction is given by a Linear Minimization Oracle (LMO) over a unit ball. This framework has emerged as a powerful family of methods for training deep networks, with recent successes including Kimi K2~\citep{kimiteam2025kimik2} and Moonlight~\citep{liu2025muonscalablellmtraining}. It unifies several modern optimizers, including normSGD~\cite{hazan2015beyond}, signSGD~\citep{bernstein2018signsgd}, Lion~\citep{chen2023lion}, and Muon~\citep{jordan2024muon}. Specifically, the LMO-based update rule is:
\begin{align}
\label{eq:update_rule_1}
&x^{t+1} = x^t + \eta^t \operatorname{LMO}(g^t), \quad \operatorname{LMO}(g^t) := \mathrm{arg~min}_{q \in \mathcal{X}: \|q\|=1} \langle g^t, q \rangle,
\end{align}
where $t$ is the optimization step, $\eta^t>0$ is the learning rate, $g^t$ is a gradient approximation (e.g., momentum), and $\| \cdot \|$ refers to an arbitrary, possibly non-Euclidean norm. This formulation arises naturally from minimizing a quadratic approximation of the loss function around the point $x^t$:
\begin{equation}
\label{eq:quadratic_approx}
f(x^t + \Delta x^t) \approx f(x^t) + \langle \nabla f(x^t), \Delta x^t \rangle + \frac{\lambda}{2} \|\Delta x^t\|^2.
\end{equation}
The update $\Delta x^t = x^{t+1} - x^t$ from \eqref{eq:update_rule_1} is the $\arg \min$ of~\eqref{eq:quadratic_approx} with respect to $\Delta x^t \in \mathcal{X}$ up to multiplication factors. Different choices of norms $\|\cdot\|$ yield different optimizers: the Euclidean norm recovers normSGD, the $\ell_1$ norm gives signSGD, and the spectral norm leads to Muon. For a detailed derivation of the resulting updates see Appendix~\ref{app:optimizers}.

The success of LMO-based methods depends not only on the appropriate choice of the norm in~\eqref{eq:update_rule_1}, but also on the proper selection of the learning rate $\eta^t$~\citep{goyal2017accurate}. In practice, empirically designed schedules are commonly used, such as linear warm-up followed by cosine decay~\citep{loshchilov2017sgdr}. In this work, we focus on the warm-up phase: starting with small learning rates and gradually increasing them before decay. Although warm-up has become nearly ubiquitous in practice~\citep{goyal2017accurate,vaswani2017attention,loshchilov2017sgdr}, its theoretical necessity has not been fully understood. 
Therefore, in this paper we address the following research questions:

\textit{(i) Can learning rate warm-up be theoretically justified for LMO-based optimizers, rather than being treated as a purely empirical heuristic?}

\textit{(ii) Can the warm-up duration be determined adaptively during training, eliminating the need for manual tuning?}

Guided by this research questions, we make the following contributions:

\begin{itemize}
\item We introduce a new generalized smoothness assumption where local curvature decreases with the suboptimality gap, and empirically verify that this behavior holds along optimization trajectories.

\item We provide a theoretical analysis establishing convergence guarantees for LMO-based optimizers under this assumption, where warm-up followed by decay emerges naturally from the proof rather than being imposed heuristically.

\item Based on the theory, we develop a practical learning rate scheduler with adaptive warm-up that relies only on standard hyperparameters and automatically determines the warm-up duration at the beginning of training.

\item We validate our approach on language model pretraining with LLaMA architectures, showing that the proposed adaptive warm-up matches or outperforms hand-tuned schedules across Muon, Lion, and normalized SGD without hyperparameter search.
\end{itemize}

\section{Related Work}

\textbf{Norm-constrained optimizers.}
Norm-constrained optimizers have recently attracted significant attention in deep learning. One of the most prominent examples of such optimizers is Muon~\citep{jordan2024muon}, which demonstrates strong performance for training deep neural networks \citep{liu2025muonscalablellmtraining}. Numerous studies have developed practical variants of Muon and related LMO-based algorithms for large-scale models, analyzing their empirical behavior under spectral or orthogonality constraints \citep{pethick2025scion,riabinin2025gluon,amsel2025polar,huang2025limuon,kovalev2025orthogonalization,he2025lowrank}.

\textbf{Adaptive and parameter-free optimizers.}
A related line of work focuses on designing optimizers that require minimal or no hyperparameter tuning. In this domain, Adam~\citep{kingma2015adam} and its variant AdamW~\citep{loshchilov2019decoupled} are adaptive coordinate-wise optimizers that have long been considered the default choice in deep learning. More recent advancements include D-Adaptation~\citep{defazio2023dadaptation}, which automatically estimates the distance to the solution to set learning rates, and Prodigy~\citep{mishchenko2024prodigy}, which improves on this with tighter distance estimates and faster adaptation. Furthermore, the recently introduced Schedule-Free methods~\citep{defazio2024schedulefree} eliminate the need for learning rate schedules entirely by maintaining iterate averages that converge without explicit decay. However, despite these advances, all discussed methods, even the Schedule-Free approaches, still rely on heuristically defined learning rate warm-up phases.

\textbf{Learning rate warm-up.}
The learning rate warm-up is a widely used heuristic to train deep neural networks, dating back at least to \citet{he2016deep}, which used a small constant learning rate during the initial training phase. The linear warm-up strategy, introduced by \citet{goyal2017accurate}, has since become standard for training ResNets~\citep{he2016deep} and transformers~\citep{vaswani2017attention}. Empirical studies have shown that warm-up enhances training stability to allow larger learning rates, reduces gradient variance, and improves model performance~\citep{gotmare2018closer,liu2019variance,kosson2024rotational}. From a geometric perspective, \citet{gilmer2021loss} and \citet{kalra2024warmup} observed that warm-up induces a sharpness reduction phase where the largest Hessian eigenvalue decreases, enabling larger learning rates in subsequent training.

\textbf{Generalized smoothness and warm-up theory.}
The standard $L$-smoothness assumption $\|\nabla f(x) - \nabla f(y)\| \le L\|x-y\|$ is insufficient to explain the necessity of warm-up, as it implies a uniform bound on curvature throughout the training landscape. Similarly, the $(L_0, L_1)$-smoothness condition introduced by \citet{zhang2020why}, which bounds the Hessian norm by $L_0 + L_1\|\nabla f(x)\|$, has not, to the best of our knowledge, yielded theoretical justifications for warm-up strategies. Recent works address this limitation by linking smoothness to the suboptimality gap, bounding the Hessian by $K_0 + K_{1} (f(x)-f^{\star})$~\citep{alimisis2025warmup} or $K_0 + K_{\rho} (f(x)-f^{\star})^{\rho}$~\citep{liu2025theoreticalwarmup}, where $f^\star$ is a target (perfect) loss for the problem \eqref{main_problem}. However, the theoretical results in these studies are limited to GD and SGD optimizers and derive schedules consisting solely of a warm-up phase without any decay. In contrast, our Assumption~\ref{ass:smoothness2} adopts this suboptimality-dependent framework to naturally derive warm-up and subsequent decay phases specifically for LMO-based methods.

\textbf{Target loss estimation.}
As was mentioned in the previous paragraph, all the theoretical frameworks about warmup rely on knowledge of the target value $f^\star$.
It commonly appears in adaptive stepsize methods, including Polyak-type step sizes \cite{polyak1963gradient, orabona2025new}, its stochastic variant \cite{loizou2021stochastic} and Polyak step method for mirror descent \cite{you2022minimizing}. 
A range of techniques has been proposed to estimate or adapt such target values in practice, including adaptive running lower bounds constructed from past function values \cite{hazan2019revisiting} and level-type schemes that replace $f^\star$ by evolving target levels in Polyak-type step sizes \cite{you2022two}.
In this work, we do not employ these mechanisms and instead fix a reasonable estimate of $f^\star$ for each setup.
In Section~\ref{sec:f_star_abl}, we perform an ablation over $f^\star$ and show that choosing $f^\star$ within a reasonable neighborhood of the optimal loss yields stable and robust behavior, consistently outperforming the best manually tuned warm-up schedules.

\section{Theoretical Analysis}

\subsection{Assumptions and Practical Motivation}
\label{sec:ass}

We begin our analysis with the set of assumptions required to study LMO-based optimizers for problem \eqref{main_problem}.

\begin{assumption}
\label{ass:star_convexity}
The function $f : \mathcal{X} \to \mathbb{R}$ is star-convex, i.e., the following inequality holds:
\[
f(\beta x^{\star} + (1 - \beta)x) \leq \beta f(x^{\star}) + (1 - \beta)f(x),
\]
for all $x \in \mathcal{X}$ and $\beta \in [0, 1]$, where $x^{\star}$ is a global minimizer of $f$, i.e. $f(x^\star) = f^\star$.
\end{assumption}

Unlike classical convexity, which requires the inequality to hold for every pair of points in $\mathcal{X}$, star-convexity only enforces convexity with respect to the optimum. This reflects a setting in which the loss landscape may be highly nonconvex globally, yet becomes progressively well behaved along trajectories that approach $x^\star$.
Such star-shaped conditions are now standard in modern theoretical analyses of deep neural networks~\citep{zhou2019sgd}. In particular, recent works analyzing LMO-based optimizers show that star-convexity, together with classical Lipschitz smoothness, is sufficient to guarantee convergence with a constant learning rate~\citep{pethick2025scion, kovalev2025orthogonalization}. However, since these works rely on the standard smoothness assumption, they fail to explain the emergence of learning rate warm-up. This motivates Assumption~\ref{ass:smoothness2}, which we introduce next.
\begin{assumption}
\label{ass:smoothness2}
The function $f : \mathcal{X} \to \mathbb{R}$ is $(\rho, K_0, K_1, K_\rho)$-smooth, i.e., there exist $K_0, K_1, K_\rho \ge 0$ and $\rho>0$ such that for all $x, y \in \mathcal{X}$ it holds that:
\begin{gather*}
    \|\nabla f(x) - \nabla f(y)\|_{\star} \le \mathcal{K}(x) \|x - y\|, \\[1ex]
    \mathcal{K}(x) := K_0 + K_1 \left( f(x) - f^{\star} \right) + K_\rho \left( f(x) - f^{\star} \right)^\rho,
\end{gather*}
where $\| \cdot \|_{\star}$ is the conjugate norm for $\| \cdot \|$ used in the update rule \eqref{eq:update_rule_1}.
\end{assumption}

This condition strengthens the classical smoothness model in a manner that is aligned with the geometry of deep learning. It recovers the $(\rho, L_0,L_{\rho})$-smoothness (see Section~\ref{sec:connection_with_smoothness} for details). When $K_{\rho}=0$ and $\|\cdot\|:=\|\cdot\|_2$, Assumption~\ref{ass:smoothness2} reduces to the version of assumption from \cite{alimisis2025warmup}, which corresponds to a bound of the type:
\[
\|\nabla f(x)-\nabla f(y)\| \le \left(K_0+K_{1}(f(x)-f^\star)\right)\|x-y\|.
\]

However, our findings (both theoretical and empirical) indicate that the regime $K_\rho>0$ is not merely an analytical convenience but necessary for accurately modeling the behavior of norm-constrained optimizers in deep learning. To demonstrate this, we evaluate the empirical stability ratio
\[
\mathcal{K}^t =
\frac{\|\nabla f(x^{t+1}) - \nabla f(x^{t})\|_\star}{\|x^{t+1}-x^t\|}
\]
as a function of $\Delta^t := f(x^t)-f^\star$
along optimization trajectories of several LMO-based methods: Muon, normSGD, and Lion. The experiments are carried out in a large-scale pretraining regime (see Section~\ref{sec:exp} for details). Figure~\ref{fig:smoothness-lion} reports the observed dependence for Lion. Analogous curves for the remaining methods are provided in Appendix~\ref{app:smooth}.

\begin{wrapfigure}[20]{r}{0.5\textwidth}
    \centering
    \includegraphics[width=0.5\textwidth]{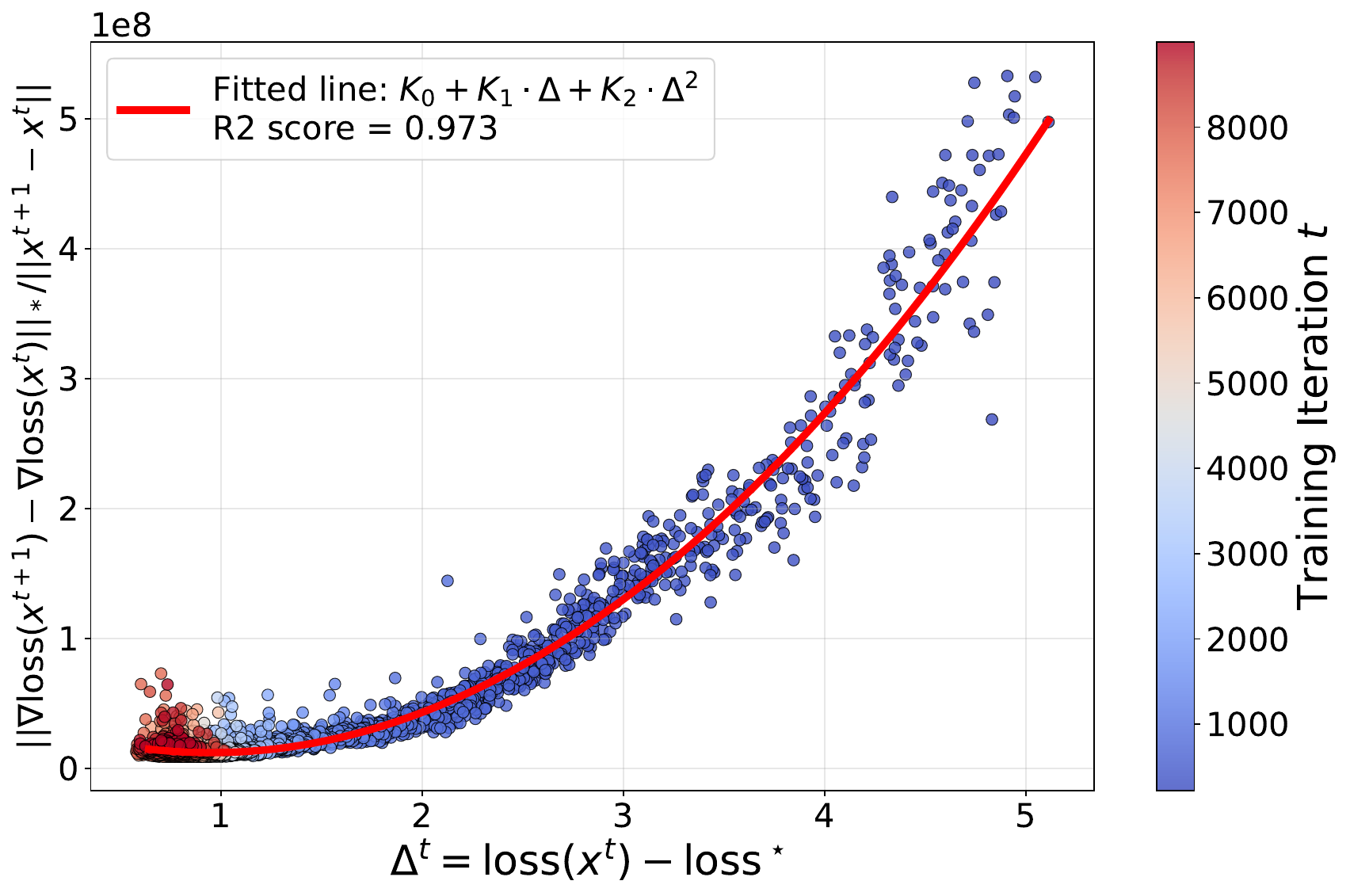}
    \caption{Empirical smoothness ratio $\mathcal{K}^t$ versus suboptimality gap $\Delta^t$ for Lion on large-scale pretraining. The trajectory is well-fitted by a quadratic dependence, indicating $K_\rho > 0$. All hyperparameter setup is provided in Appendix \ref{appendix:hyper_ass}.}
    \label{fig:smoothness-lion}
\end{wrapfigure}

In all cases we observe a clear parabolic trend linking $\mathcal{K}^t$ and $\Delta_t$, consistent with a dominant $K_\rho(f(x)-f^\star)^\rho$ term. This behavior persists across optimizers, and cannot be captured by models with $K_\rho=0$.

We note that earlier empirical studies generated smoothness-versus-progress plots and fitted them using linear functions~\cite{alimisis2025warmup}. However, those investigations examined only the earliest portion of training, under extremely small learning rates, where the suboptimality gap remains large (e.g., $\Delta \geq 6$). In that narrow range, a linear fit indeed appears adequate. Yet, over the full training trajectory, where $\Delta$ decreases by orders of magnitude, a purely linear model breaks down, and the quadratic component becomes essential.
In recent warm-up studies, \citet{liu2025theoreticalwarmup} also observed that linear functions may not fully capture these empirical trends. While their analysis effectively identifies this non-linearity, it primarily focuses on diagnosing the limitations of linear models. Our work builds upon these observations by showing that a quadratic model provides a more precise functional form to describe these dynamics across the entire trajectory, further supporting Assumption~\ref{ass:smoothness2} as a representative model for deep learning.

Assumption~\ref{ass:smoothness2} is stated globally as a sufficient condition for the proof. In practice it suffices that the structure holds along the LMO trajectory, which Figure~\ref{fig:smoothness-lion} validates directly.
The ratio $\mathcal{K}^t$ differs from the top Hessian eigenvalue studied in Edge of Stability~\citep{cohen2021gradient}: both can coexist, and no EoS analog is known for LMO methods such as Muon. Appendix~\ref{app:smooth} details the estimation procedure, provides validation for all three optimizers, and discusses the EoS relation in Section~\ref{app:eos}.

Finally, we require the generated sequence of iterates to remain within a bounded domain.

\begin{assumption}
\label{ass:boundedness}
The iterates $x^t$ generated by \eqref{eq:update_rule_1} are bounded, i.e., there exists $D>0$ such that for all $t \ge 0$ it holds that
$
\|x^t - x^\star\| \le D.
$\end{assumption}

Assumption~\ref{ass:boundedness} is standard in LMO-type methods~\citep{kovalev2025orthogonalization}, ensuring that the smoothness constants in Assumption~\ref{ass:smoothness2} remain meaningful along the trajectory. Importantly, this assumption is not required in two common cases: (i)~when weight decay is used (Section~\ref{sec:weight_decay}), (ii)~for the Euclidean norm in the step \eqref{eq:update_rule_1}, boundedness follows automatically (see Lemma~\ref{lem:boundedness} in Appendix).

\subsection{Deterministic Case without Weight Decay}

The next theorem provides, to the best of our knowledge, the first convergence analysis under Assumption~\ref{ass:smoothness2} to theoretically predict a learning rate schedule consisting of a warm-up phase followed by decay.

\begin{theorem} 
\label{thm:1}
Suppose Assumptions \ref{ass:star_convexity}, \ref{ass:smoothness2}, and \ref{ass:boundedness} hold and the iterates $x^t$ are generated by \eqref{eq:update_rule_1} with $g^t = \nabla f(x^t)$.

If we use the learning rate scheduler (with warm-up followed by decay for $\rho>1$)
\[
\eta^{t}=\frac{\Delta^{t}}{D \cdot \mathcal{K}(x^t)} = \frac{\Delta^{t}}{D \cdot \left( K_0 + K_1 \Delta^t + K_\rho (\Delta^t)^\rho \right)},
\]
then $\Delta^{t+1}\le\Delta^{t}$ and $\mathcal{K}(x^{t+1}) \leq \mathcal{K}(x^t)$ for all $t$, and
\[
\Delta^T \le \frac{2D^2\sum_{t=0}^{T-1}\mathcal{K}(x^t)}{T^2} = \mathcal{O}\left( \frac{1}{T} \right),
\]
where $\Delta^t:=f(x^t)-f^{\star}$.
\end{theorem}

\textbf{Discussion of Theorem~\ref{thm:1}.}

\begin{wrapfigure}[17]{r}{0.45\textwidth}
    \centering
    \includegraphics[width=0.45\textwidth]{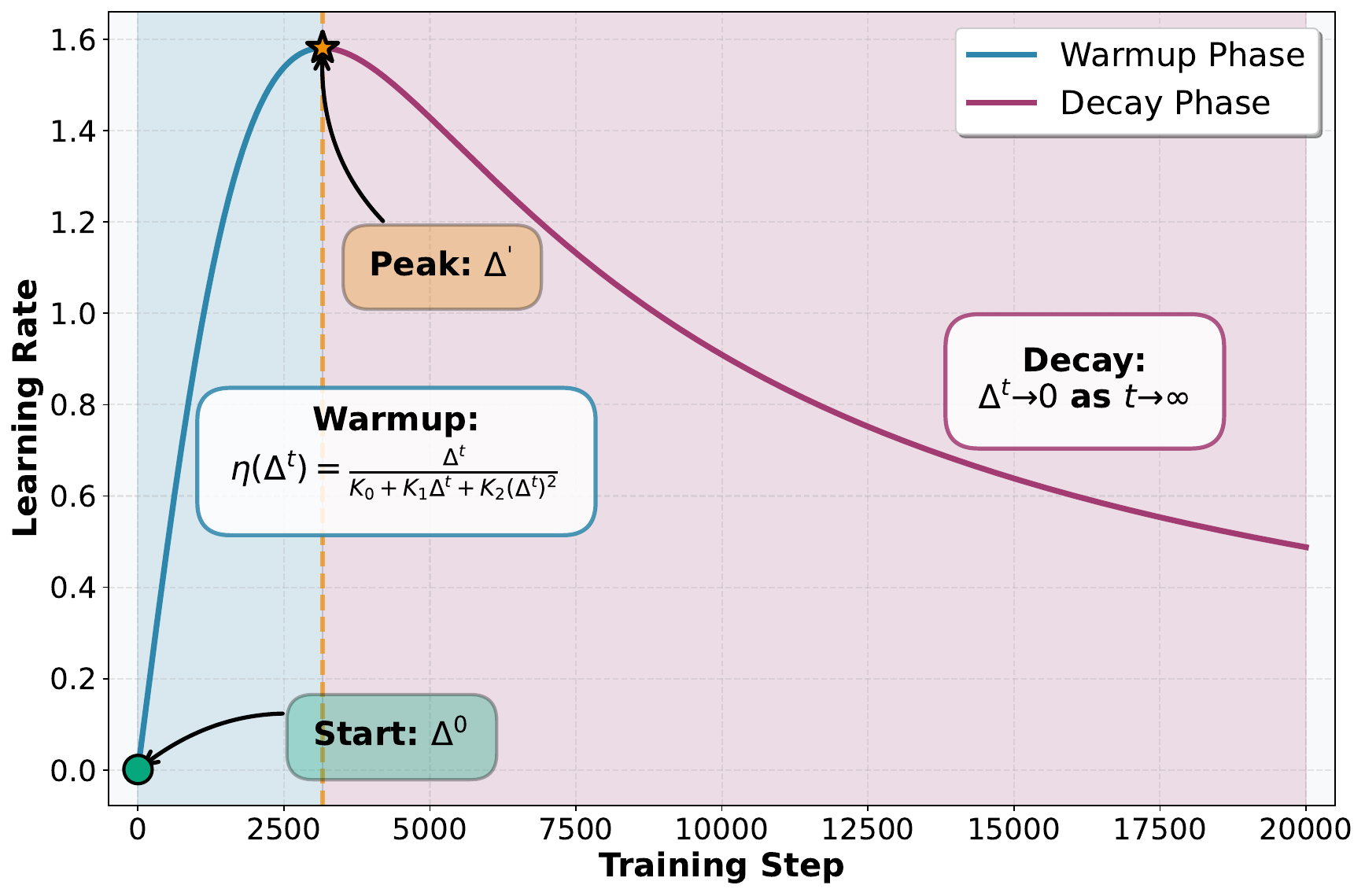}
    \caption{Asymptotic behavior of the learning rate $\eta^t$ for parameters $D=1$, $\rho=2$, $K_0=10^{-4}$, $K_1=0$ and $K_\rho=10^3$ with $\Delta^t = 1 / t$. The trajectory explicitly shows the theoretical warm-up phase followed by decay.}
    \label{fig:eta_behavior}
\end{wrapfigure}

$\bullet$ \textbf{Explanation of warm-up behavior.}
Since by Theorem~\ref{thm:1} $\Delta^t$ decreases monotonically, the learning rate $\eta^t$ exhibits a warm-up phase followed by a decay for $\rho > 1$. The transition occurs at the point $\Delta'$ obtained by maximizing $\eta(\Delta^t)$: $\Delta' := \left( \frac{K_0}{K_\rho (\rho - 1)} \right)^{1/\rho}$.
Consequently, $\eta^t$ increases while $\Delta^t > \Delta'$ and decreases thereafter. If we assume sublinear convergence of the objective function, specifically $\Delta^t \sim 1/t$, the asymptotic behavior of the learning rate is illustrated in Figure~\ref{fig:eta_behavior}.

$\bullet$ \textbf{Comparison with decay-only schedule.}
Consider a simplified scheduler where the adaptive term $\mathcal{K}(x^t)$ is replaced by its initial constant value $\mathcal{K}^0:=K_0 + K_1 \Delta^0 + K_\rho (\Delta^0)^\rho$. In this scenario, the learning rate becomes $\eta^{t}=\frac{\Delta^{t}}{D \cdot \mathcal{K}^0}$ and the convergence bound simplifies to
\[
\Delta^T \le \frac{2D^2 T \cdot \mathcal{K}^0}{T^2} = \frac{2D^2 \mathcal{K}^0}{T} = \mathcal{O}\left( \frac{1}{T} \right).
\]
Although the asymptotic rate $\mathcal{O}(1/T)$ remains unchanged, the adaptive schedule in Theorem~\ref{thm:1} provides a tighter bound. To quantify the improvement, observe that if $\Delta^t \sim 1/t$, then $\mathcal{K}(x^t) \sim K_0 + K_1/t + K_\rho /t^{\rho}$. For $\rho > 1$, the dominant term is $K_0$, yielding $\sum_{t=1}^{T} \mathcal{K}(x^t) \sim K_0 T$. Thus, the adaptive bound $\Delta^T \lesssim D^2 K_0 / T$ matches the decay-only bound asymptotically. However, during the transient phase (small $t$), the higher-order terms $K_1/t$ and $K_\rho/t^{\rho}$ are significant, and the adaptive schedule benefits from using larger learning rates.

\subsection{Deterministic Case with Weight Decay}
\label{sec:weight_decay}

We now extend our framework by incorporating weight decay regularization with the coefficient $\lambda > 0$ into \eqref{eq:update_rule_1}, resulting in the following update rule:
\begin{equation}
\label{eq:update_rule_2}
x^{t+1} = (1-\lambda \eta^t)x^t + \eta^t \operatorname{LMO}(g^t).
\end{equation}
This form of weight decay is a standard technique for improving generalization in large-scale neural networks \citep{loshchilov2019decoupled}. The following Theorem~\ref{thm:2} provides the convergence guarantees for the step \eqref{eq:update_rule_2}.

\begin{theorem} 
\label{thm:2}
Suppose Assumptions \ref{ass:star_convexity} and \ref{ass:smoothness2} hold with $\rho>1$ and the iterates $x^t$ are generated by \eqref{eq:update_rule_2} with $g^t = \nabla f(x^t)$, 
$\lambda \in \big(0, \frac{1}{\max\left(\|x^0\|, \|x^{\star}\|, 1/\lambda_{\max}\right)} \big]$,
where $$\lambda_{\max}=\left[ 8 \left( \rho \left( \frac{K_0}{\rho - 1} \right)^{\frac{\rho - 1}{\rho}} K_{\rho}^{\frac{1}{\rho}} + K_1 \right) \right]^{1/2}.$$

If we use the learning rate scheduler (with warm-up followed by decay)
\[
\eta^{t}=\frac{\lambda \Delta^{t}}{8 \mathcal{K}(x^t)}= \frac{\lambda \Delta^{t}}{8 \left( K_0 + K_1 \Delta^t + K_\rho (\Delta^t)^\rho \right)},
\]
then $\Delta^{t+1}\le\Delta^{t}$ and $\mathcal{K}(x^{t+1}) \leq \mathcal{K}(x^t)$ for all $t$, and
\[
\Delta^T \le \frac{16\sum_{t=0}^{T-1}\mathcal{K}(x^t)}{\lambda^2T^2} = \mathcal{O}\left( \frac{1}{T} \right).
\]
\end{theorem}

\textbf{Discussion of Theorem~\ref{thm:2}.}

\begin{itemize}
\item \textbf{Relaxation of Boundedness Assumption.}
Unlike Theorem~\ref{thm:1}, we do not require Assumption~\ref{ass:boundedness} (boundedness of the iterates), as the regularization term $-\lambda \eta^t x^t$ in the update rule \eqref{eq:update_rule_2} implicitly ensures that the iterates remain bounded.

\item \textbf{Comparison with Theorem~\ref{thm:1}.} Note that all implications of Theorem~\ref{thm:1} are satisfied in the context of Theorem~\ref{thm:2} by setting $D \sim 1/\lambda$.
\end{itemize}

The deterministic analysis provides the foundation for understanding optimal learning rate schedules. However, as deep learning relies on stochastic gradient estimates in practice, we now extend our framework to address it.

\section{Stochastic Extensions}

Let us now consider the following stochastic optimization problem:
\begin{equation}
\label{eq:stoch_opt_problem}
\min_{x \in \mathcal{X}} \left\{ f(x) := \mathbb{E}_{\xi \sim \mathcal{D}} [f_{\xi}(x)] \right\},
\end{equation}
where $f_{\xi}: \mathcal{X} \to \mathbb{R}$ represents the loss of the model parameterized by $x$, associated with a training data point $\xi$ sampled from the probability distribution $\mathcal{D}$.

We consider a variant of the update rule \eqref{eq:update_rule_1} for minimizing \eqref{eq:stoch_opt_problem}, by setting $g^t = \nabla f_{\xi^t}(x^t)$ and choosing the Euclidean norm $\|\cdot\| := \|\cdot\|_2$ in the LMO step:
\begin{equation}
\label{eq:update_rule_3}
x^{t+1}=x^{t}-\eta^t\frac{\nabla f_{\xi^t} (x^t)}{\|\nabla f_{\xi^t} (x^t)\|}.
\end{equation}

For the stochastic setting, we can use a weaker variant of Assumption \ref{ass:smoothness2} that only requires smoothness relative to the optimum $x^{\star}$.


\begin{assumption}
\label{ass:smoothness2_weak}
The functions $f_\xi : \mathcal{X} \to \mathbb{R}$ are $(\rho, K_0, K_1, K_\rho)$-smooth at the global minimizer $x^{\star}$ of $f$ for the Euclidean norm, i.e., there exist $K_0, K_1, K_\rho \ge 0$, $\rho>0$ such that for all $x \in \mathcal{X}$ it holds that:
\begin{gather*}
    \|\nabla f_\xi(x)\|_{2} \le \mathcal{K}_\xi(x) \|x - x^{\star}\|_2, \\[1ex]
    \mathcal{K}_\xi(x) := K_0 + K_1 \left( f_\xi(x) - f^{\star}_\xi \right) + K_\rho \left( f_\xi(x) - f^{\star}_\xi \right)^\rho.
\end{gather*}
\end{assumption}

Additionally, to study convergence in the stochastic setting similarly to \citet{alimisis2025warmup}, we require an interpolation condition, which is typically satisfied for over-parameterized networks \citep{pmlr-v80-ma18a}.

\begin{assumption}
\label{ass:overparam}
The optimization problem \eqref{eq:stoch_opt_problem} satisfies the interpolation regime, i.e., for the global minimizer $x^{\star}$ of $f$, the following holds almost surely for $\xi \sim \mathcal{D}$
\[
f_{\xi}(x^{\star}) = f_{\xi}^{\star},
\]
where $f_{\xi}^{\star} := \inf_{x \in \mathcal{X}} f_{\xi}(x) > -\infty$.
\end{assumption}

Under the established assumptions, we characterize the convergence of \eqref{eq:update_rule_3} in Theorem~\ref{thm:3}.

\begin{theorem} 
\label{thm:3}
Suppose Assumptions \ref{ass:star_convexity} (for all $f_\xi$), \ref{ass:smoothness2_weak} and \ref{ass:overparam} hold. 
Consider the iterates $x^t$ generated by \eqref{eq:update_rule_3} with
\[
\eta^{t}=\frac{\Delta_{\xi}^{t}}{D \cdot \mathcal{K}_\xi(x^t)},
\]
where $D := \|x^0 - x^\star\|$, $\Delta^t_{\xi}:=f_{\xi^t}(x^t) - f_{\xi^t}^{\star}$, and $\mathcal{K}_\xi(x^t) := K_0 + K_1 \Delta_{\xi}^t + K_\rho (\Delta_{\xi}^{t})^{\rho}$, then
\[
\frac{1}{T}\sum_{t=0}^{T-1}\mathbb{E}[\Delta_{\xi}^t]
\le \frac{D^2}{T}\sqrt{\mathbb{E}\left[\sum_{t=0}^{T-1} (\mathcal{K}_\xi(x^t))^2\right]}.
\]

If there exists $M>0$ such that $\Delta_{\xi}^t \le M$ almost surely for all $t$ and $\xi \sim \mathcal{D}$, then
\[
\frac{1}{T}\sum_{t=0}^{T-1}\mathbb{E}[\Delta_{\xi}^t]
\le \frac{D^2 \bar{\mathcal{K}}}{\sqrt{T}} = \mathcal{O}\left(\frac{1}{\sqrt{T}}\right),
\]
where $\bar{\mathcal{K}} := K_0 + K_1 M + K_\rho M^{\rho}$.
\end{theorem}

\section{Practical Scheduler with Adaptive Warmup}
\label{sec:our_sceduler}
In this section, we demonstrate how to apply the optimal learning rate scheduler in practice. Fixing $\rho = 2$ and removing constant scaling factors for clarity, the scheduler form Theorems \ref{thm:1}, \ref{thm:2} and \ref{thm:3} takes the form
\begin{equation}
    \label{eq:practical_eta}
    \eta(\Delta) = \frac{\Delta}{K_0 + K_1 \Delta + K_2 \Delta^2}.
\end{equation}
Since the coefficients $K_0,K_1,K_2$ are unknown and cannot be measured during training, we determine them through three (by the number of the independent parameters) practical constraints. 

\paragraph{(1) Peak learning rate.}
The function $\eta(\Delta)$ has a unique local maximum at some $\Delta' \in [0; \Delta^0 = f(x^0) - f^\star]$, and we enforce
\begin{equation}
    \label{eq:constr_1}
    \eta(\Delta') = \mathrm{lr},
\end{equation}
where $\mathrm{lr}$ denotes the peak learning rate. This quantity is a standard user-facing parameter in all PyTorch learning rate schedulers~\citep{paszke2019pytorch}.

\paragraph{(2) Warm-up floor.}
We match the initial step size by requiring
\begin{equation}
    \label{eq:constr_2}
    \eta(\Delta^0) = \frac{\mathrm{lr}}{\mathrm{div}},
\end{equation}
where $\mathrm{div} \geq 1$ is a user-specified divisor. Typical implementations use $\mathrm{div} = 100$, again mirroring standard PyTorch practice~\citep{paszke2019pytorch}.

We solve constraints \eqref{eq:constr_1} and \eqref{eq:constr_2} analytically to express $K_0, K_1$ and $K_2$ in terms of $\Delta'$ (see Appendix~\ref{app:K_012} for closed-form expressions). Now we need to define $\Delta'$.

\paragraph{(3) Target-shape matching.}
The third condition determines $\Delta'$ and aligns the profile of \eqref{eq:practical_eta} with a classical warmup+decay schedule. We first introduce the target learning rate $\eta_{\mathrm{trgt}}(\Delta)$ that we aim to approximate:
\begin{equation*}\label{eq:target_schedule}
\eta_{\mathrm{trgt}}(\Delta) =
\begin{cases}
\displaystyle
\frac{\mathrm{lr}}{\mathrm{div}} 
+
\frac{(\mathrm{lr} - \frac{\mathrm{lr}}{\mathrm{div}})(\Delta^0 - \Delta)}{\Delta^0 - \Delta'},
\!\!\!\!&
\Delta \in [\Delta',\Delta^0],\\[10pt]
\displaystyle
\frac{1}{2} \mathrm{lr} \Bigl(1 - \cos \bigl(\pi \Delta / \Delta'\bigr)\Bigr),
&
\Delta \in [0,\Delta'].
\end{cases}
\end{equation*}
The definition of $\eta_{\mathrm{trgt}}$ is provided as the most commonly used choice in practice: linear warm-up and cosine decay \cite{paszke2019pytorch,loshchilov2017sgdr}. The functions on the intervals  $[0,\Delta']$ and $[\Delta',\Delta^0]$ may be replaced with any schedulers without affecting the construction.

Rather than matching $\eta(\Delta)$ to $\eta_{\mathrm{trgt}}(\Delta)$ over the entire interval $[0, \Delta^0]$, we focus on a neighborhood around $\Delta'$, since, this is the switching point between warmup and decay. For large deviations away from $\Delta'$, the theoretical schedule \eqref{eq:practical_eta} significantly differs from $\eta_{\mathrm{trgt}}(\Delta)$, whereas matching in a neighborhood of $\Delta'$ controls the behavior of the scheduler exactly where the transition happens. Moreover, by constraint \eqref{eq:constr_1} we have $\eta(\Delta') = \eta_{\mathrm{trgt}}(\Delta')$, therefore the approximation error vanishes at the center. The neighborhood around $\Delta'$ also corresponds to the maximal learning rate, and therefore to the strongest influence on training.

To formalize this localized matching, we introduce an exponential weight $\exp\bigl[-(\Delta - \Delta')^2 / \sigma^2 \bigr]$ that is maximized at $\Delta'$ and decays as we move away from it. The parameter $\sigma$ controls the width of the neighborhood and must be chosen appropriately. 
We determine $\sigma$ by considering the effective step size measured in Frobenius norm. Different LMO geometries produce updates of different magnitudes: for signSGD ($\ell_1$ norm), the output vector of dimensionality $d$ consists of $\pm 1$ entries and thus has squared Frobenius norm $d$, requiring a smaller learning rate than normSGD (Euclidean norm), whose output has unit Frobenius norm. We formalize this idea as follows: for a given $\Delta$ and gradient $g^t$, the squared step size $x^{t+1} - x^t$ satisfies
\begin{equation*}\label{eq:step_bound}
\bigl\|\eta(\Delta) \mathrm{LMO}(g^t)\bigr\|_F^2
=
\eta(\Delta)^2 \bigl\|\mathrm{LMO}(g^t)\bigr\|_F^2
\;\leq\;
\kappa \eta(\Delta)^2,
\end{equation*}
where
$
\kappa := \sup_{\|u\|=1} \|u\|_F^2.
$
captures the dependence on the norm used inside the LMO step \eqref{eq:update_rule_1}. Intuitively, $\kappa$ is the worst-case factor between the unit ball of the optimization norm $\|\cdot\|$ and the Frobenius norm ($\kappa = d$ for signSGD). It depends on the structure of the model: number of layers and the shapes of their parameter tensors and can be computed analytically for each geometry. In Appendix \ref{app:kappas} we provide closed-form expressions together with the corresponding derivations for $\kappa$ for all optimizers considered in our experiments.

Thus, if we think of the effective step $\sqrt{\kappa} \eta(\Delta)$ as having a fixed Frobenius-scale variance $\sigma_F^2$, then the corresponding variance in $\eta(\Delta)$ itself must be $\sigma_F^2/\kappa$. This leads to setting $\sigma^2 = \sigma_F^2/\kappa$ in the exponential weight, which normalizes the neighborhood around $\Delta'$ in a way that is consistent across different LMO geometries.

Combining these considerations, we formulate the objective function that we wish to minimize with respect to $\Delta'$:
\begin{equation}\label{eq:constr_3}
\int_{0}^{\Delta^0} \exp\left[-\frac{(\Delta - \Delta')^2 \kappa}{\sigma_F^2} \right] \bigl(\eta(\Delta) - \eta_{\mathrm{trgt}}(\Delta)\bigr)^2 d \Delta,
\end{equation}
where $\sigma_F$ is a fixed constant.
This condition is then enforced numerically: we sample $1000$ candidate values of $\Delta'$, evaluate the objective \eqref{eq:constr_3}
for each candidate, and select the minimizer. Since this procedure is executed only once at the beginning of training, it contributes no observable overhead to the overall runtime. Once $K_0$, $K_1$, $K_2$ and $\Delta'$ are determined from constraints \eqref{eq:constr_1}--\eqref{eq:constr_3}, the scheduler is fully specified. Algorithm~\ref{alg:lr_scheduler} summarizes the complete scheduler in a practical, ready-to-implement form.

\begin{algorithm}{Adaptive warmup scheduler}
\label{alg:lr_scheduler}
\begin{algorithmic}[1]
\State \textbf{Input:} total steps $T$,
target loss $f^\star$, peak $\mathrm{lr}$, divisor $\mathrm{div}$, warmup and decay schedulers = $(\mathrm{linear}, \mathrm{cos})$, optimizer $\in \{\text{Muon}, \text{Lion}, \text{normSGD}\}$.
\State \textbf{Initialize:} $\texttt{is\_init} \leftarrow \texttt{False}$, $\texttt{is\_decay} \leftarrow \texttt{False}$, $\texttt{warmup\_steps} \leftarrow 0$.
\Function{get\_lr}{$\texttt{loss}$}
    \State $\Delta \leftarrow \texttt{loss} - f^\star$
    \If{not \texttt{is\_init}} 
        \State $\texttt{is\_init} \leftarrow \texttt{True}$ 
        \State determine $K_0, K_1, K_2, \Delta'$ via constraints \eqref{eq:constr_1}–\eqref{eq:constr_3} \label{line:1}
    \EndIf 
    \If{$\Delta \geq \Delta'$ \textbf{and} not \texttt{is\_decay}} \label{line:3}
        \State $\texttt{warmup\_steps} \leftarrow \texttt{warmup\_steps} + 1$
        \State \textbf{return} $\displaystyle \frac{\Delta}{K_0 + K_1 \Delta + K_2 \Delta^2}$ \label{line:4}
    \Else \label{line:5}
        \If{not \texttt{is\_decay}}
            \State $\texttt{is\_decay} \leftarrow \texttt{True}$
            \State initialize decay scheduler with
            $T_{\text{decay}} \leftarrow T - \texttt{warmup\_steps}$ and starting lr $=\mathrm{lr}$
        \EndIf
        \State \textbf{return} $\texttt{decay\_scheduler.get\_lr}()$ \label{line:6}
    \EndIf
\EndFunction
\end{algorithmic}
\end{algorithm}

Line 7 of Algorithm \ref{alg:lr_scheduler} uses the initialization procedure discussed in this section. Lines 9-11 apply \eqref{eq:practical_eta} to compute the adaptive learning rate. When $\Delta < \Delta'$, we switch from the adaptive warmup \eqref{eq:practical_eta} to a classical decay scheduler (lines 12-17). This choice is motivated by two factors: (i) the theoretical form in~\eqref{eq:practical_eta} produces tiny learning rates at small $\Delta$, and (ii) the parabolic approximation of the empirical smoothness ratio disappears in that regime (see Figure \ref{fig:smoothness-lion}).

\begin{figure*}[ht]
    \centering
    \begin{subfigure}{0.32\textwidth}
        \includegraphics[width=\linewidth]{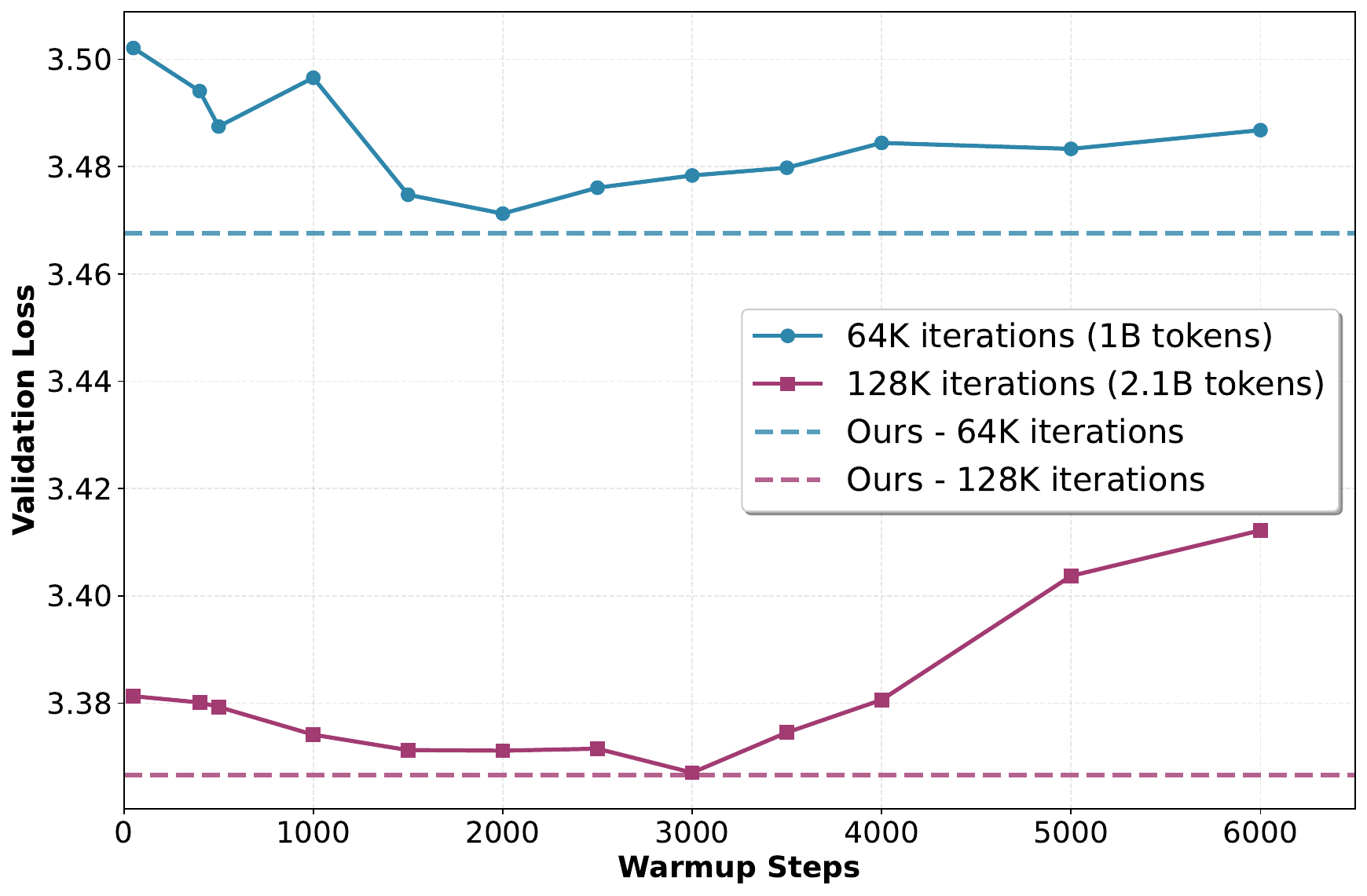}
    \end{subfigure}
    \hfill
    \begin{subfigure}{0.32\textwidth}
        \includegraphics[width=\linewidth]{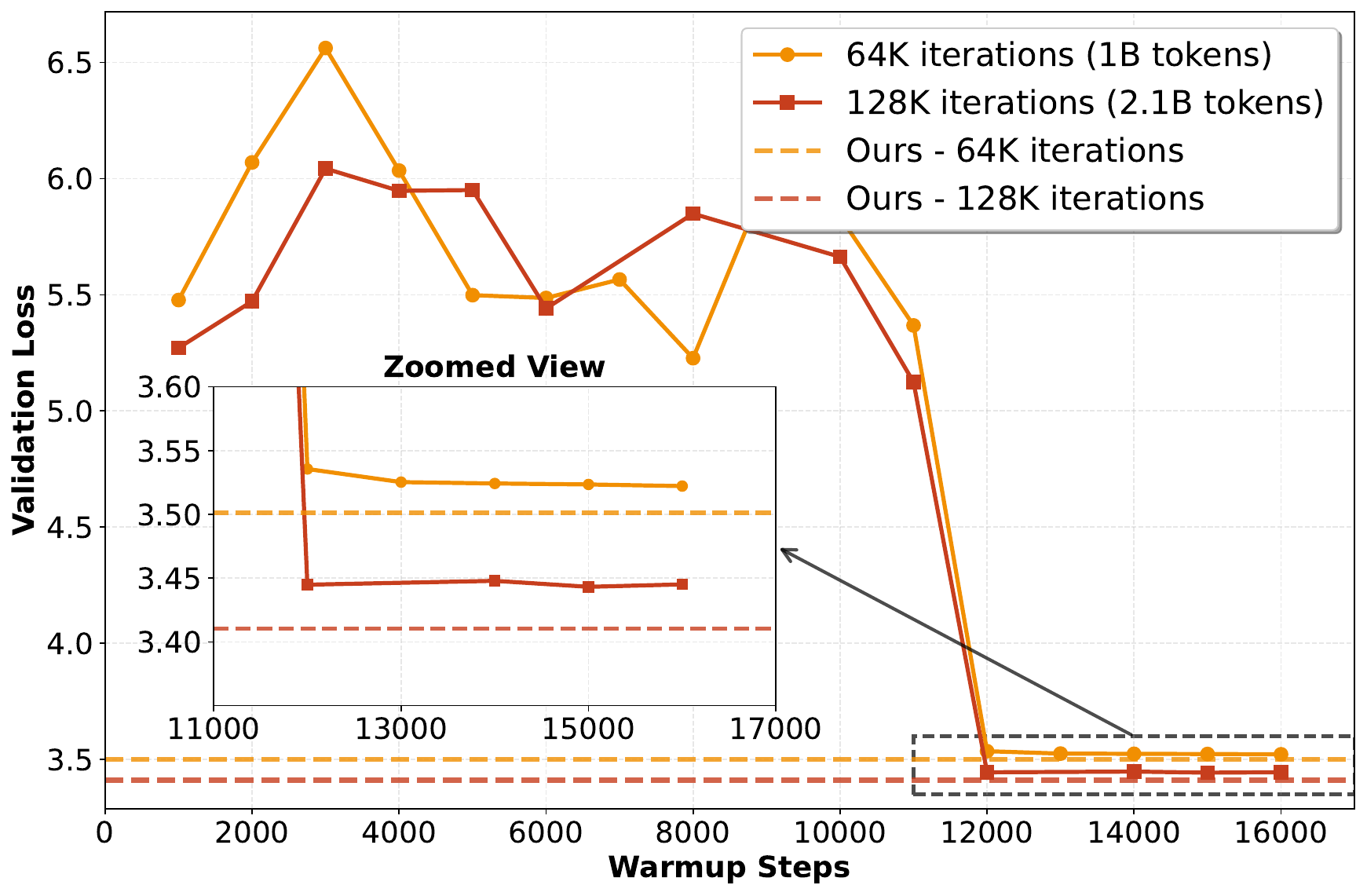}
    \end{subfigure}
    \hfill
    \begin{subfigure}{0.32\textwidth}
        \includegraphics[width=\linewidth]{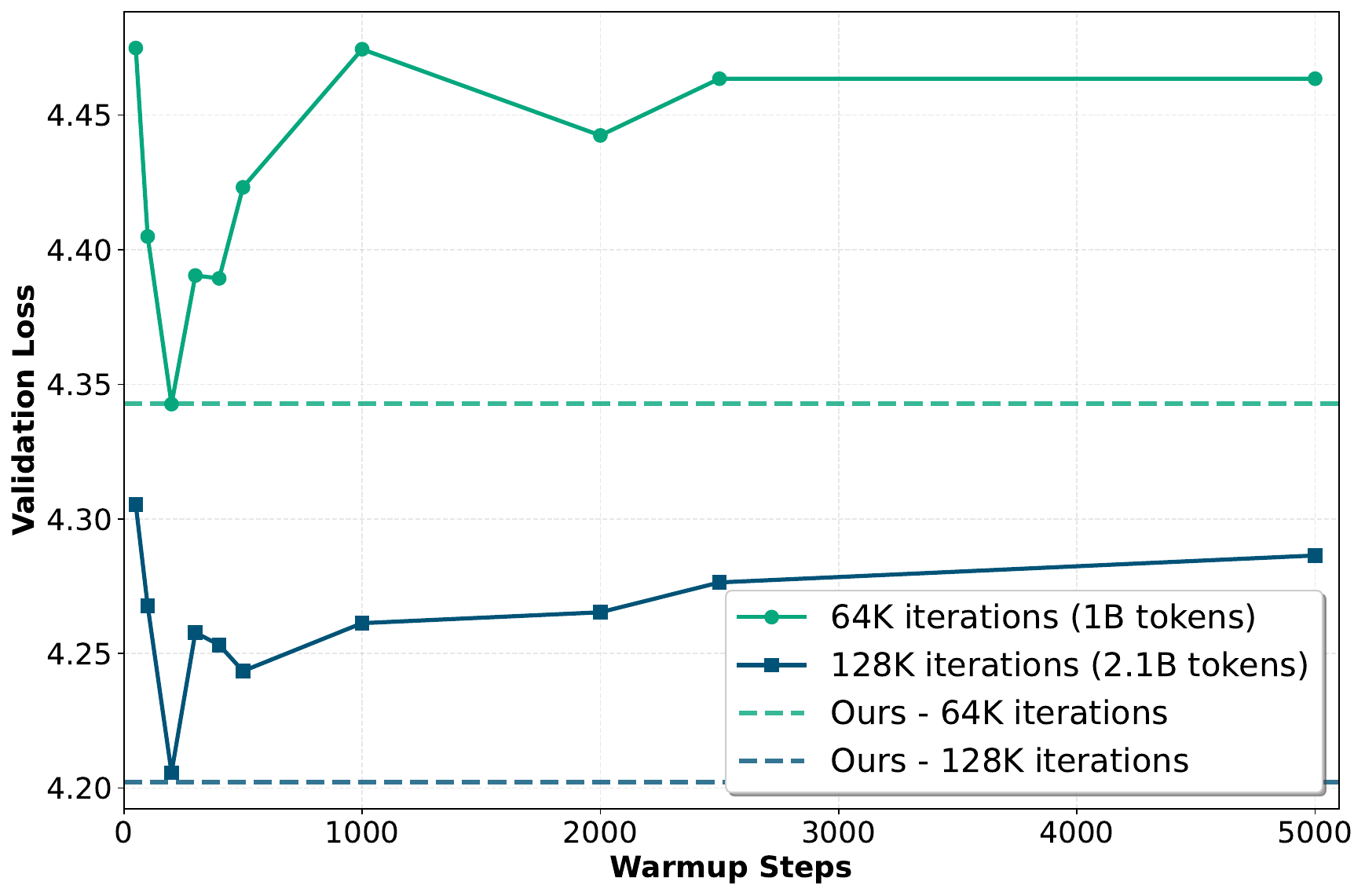}
    \end{subfigure}

    \begin{subfigure}{0.32\textwidth}
        \includegraphics[width=\linewidth]{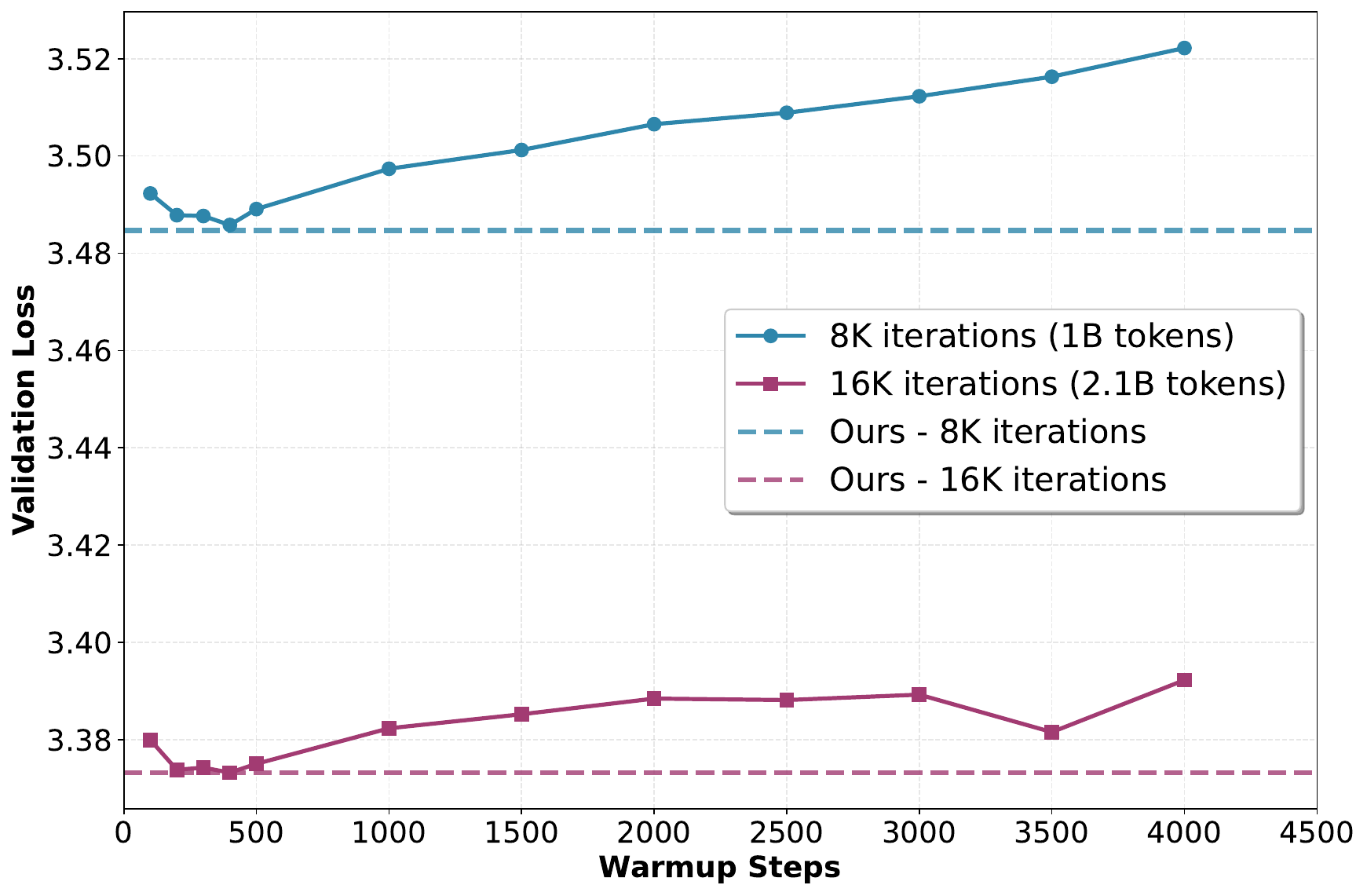}
        \caption{Muon}
    \end{subfigure}
    \hfill
    \begin{subfigure}{0.32\textwidth}
        \includegraphics[width=\linewidth]{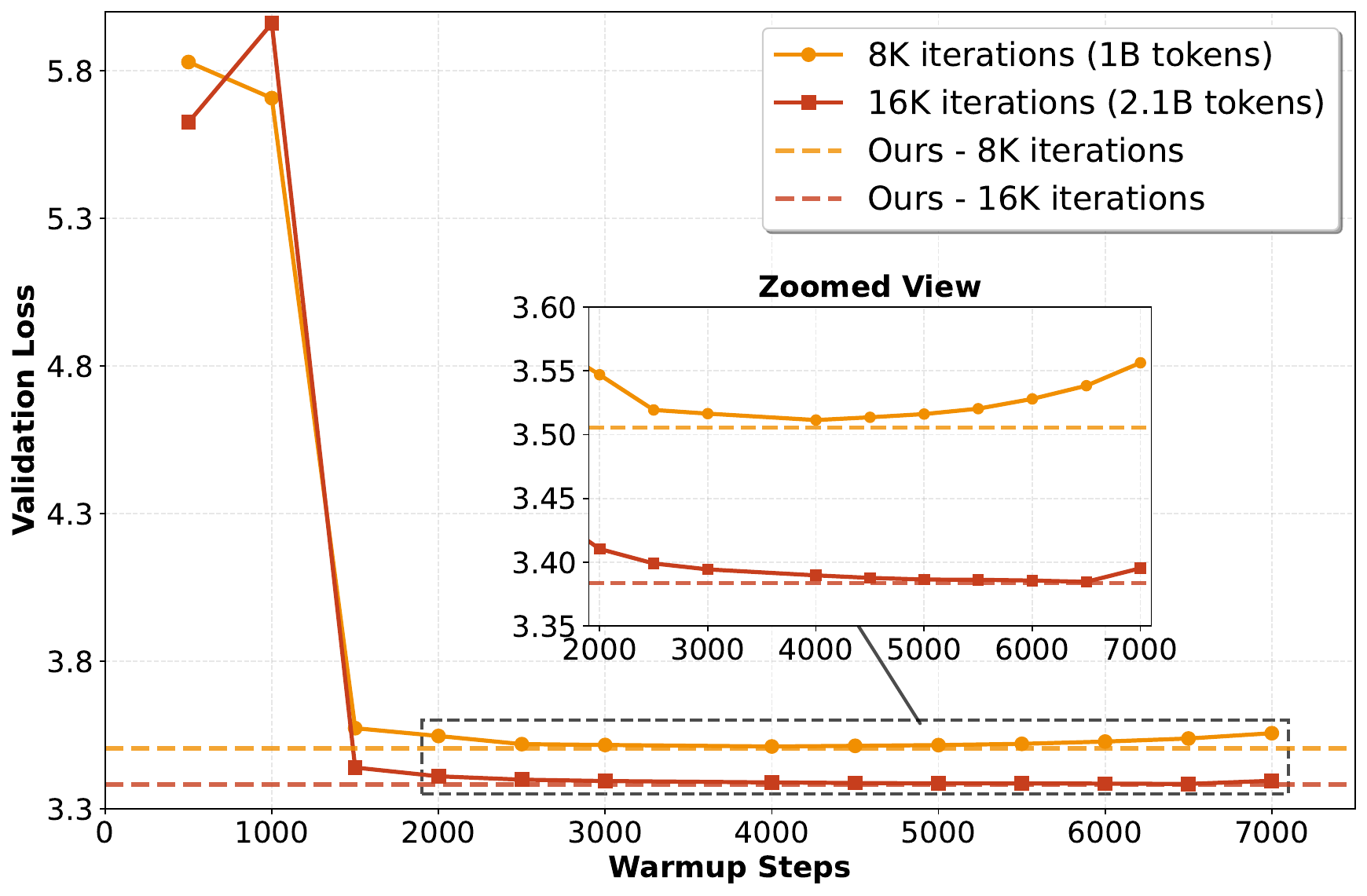}
        \caption{Lion}
    \end{subfigure}
    \hfill
    \begin{subfigure}{0.32\textwidth}
        \includegraphics[width=\linewidth]{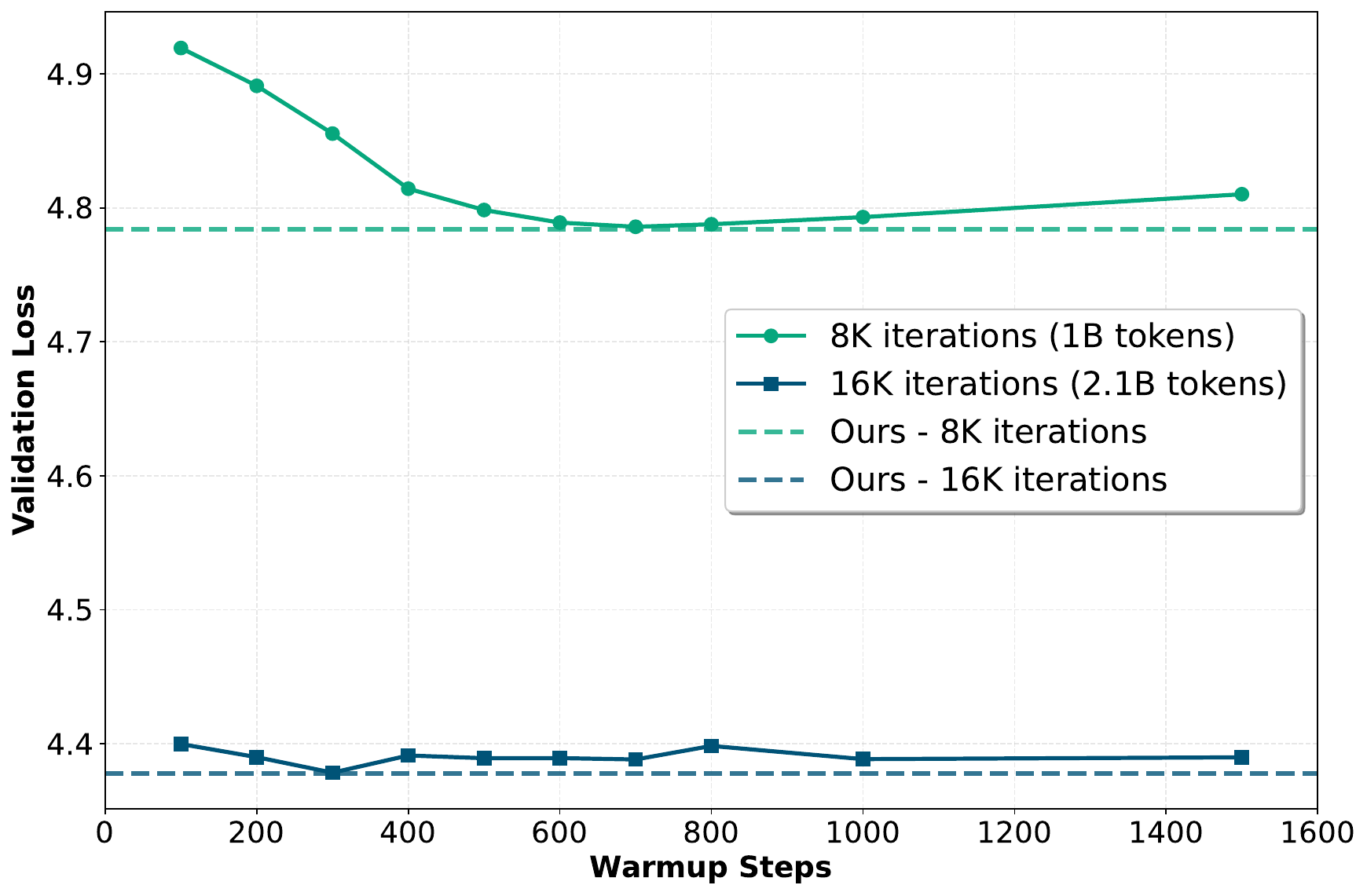}
        \caption{normSGD}
    \end{subfigure}
    \caption{Validation loss on LLaMA 124M model with $\mathrm{bs} = 32$ (top) and $256$ (bottom) as a function of manually selected warmup length (solid) vs.\ adaptive warmup (dashed). Across optimizers, the adaptive method outperforms or at least matches the best manually tuned value without any grid search. }
    \label{fig:warmup-sweep-124M}
\end{figure*}

\section{Experiments}
\label{sec:exp}

We evaluate the proposed scheduler (Algorithm \ref{alg:lr_scheduler}) in a medium- and large-scale LLM pretraining. Our setup follows the experimental protocol of~\cite{semenov2025benchmarking}, where several optimizers were benchmarked on training LLaMA-style transformer models on the FineWeb dataset \cite{penedo2024fineweb}.
For direct comparability, we adopt the same model architecture family: decoder-only transformer networks with rotary position embeddings and the standard LLaMA hyperparameter configuration \cite{fedus2022switch}.
We reuse the optimizer configurations from~\cite{semenov2025benchmarking} wherever available. Specifically, for Muon and Lion we copy all hyperparameters (learning rate, weight decay, etc.) from the reference setup. We additionally evaluate normSGD, that was not included in~\cite{semenov2025benchmarking}, therefore for this optimizer we tune  $\mathrm{lr}$ and $\mathrm{div}$ independently. Across all model sizes, batches, and optimizers, we use fixed Frobenius variance $\sigma_F^2=10^3$ without any additional tuning. Full tuning ranges and configuration tables are provided in Appendix~\ref{appendix:hyper}.
We consider models with 124M and 256M parameters, and evaluate $32$ and $256$ batch-size regimes. For each optimizer, we sweep over a range of manually specified warmup lengths and compare the resulting final validation loss against the value obtained from Algorithm~\ref{alg:lr_scheduler}. Figures~\ref{fig:warmup-sweep-124M} and \ref{fig:warmup-sweep-210M} summarize the trends for Muon, Lion and normSGD for 124M and 256M model sizes respectfully.

To further demonstrate that the proposed scheduler generalizes beyond language modeling, we additionally evaluate it on a computer vision pretraining task using a Swin Transformer Base (Swin-B)~\citep{liu2021swin} trained from scratch on ImageNet-1K~\citep{5206848} with the Muon optimizer; see Section~\ref{sec:llm_results} and Appendix~\ref{appendix:hyper_swin} for full hyperparameters.

\begin{figure*}[t]
    \centering
    \begin{subfigure}{0.32\textwidth}
        \includegraphics[width=\linewidth]{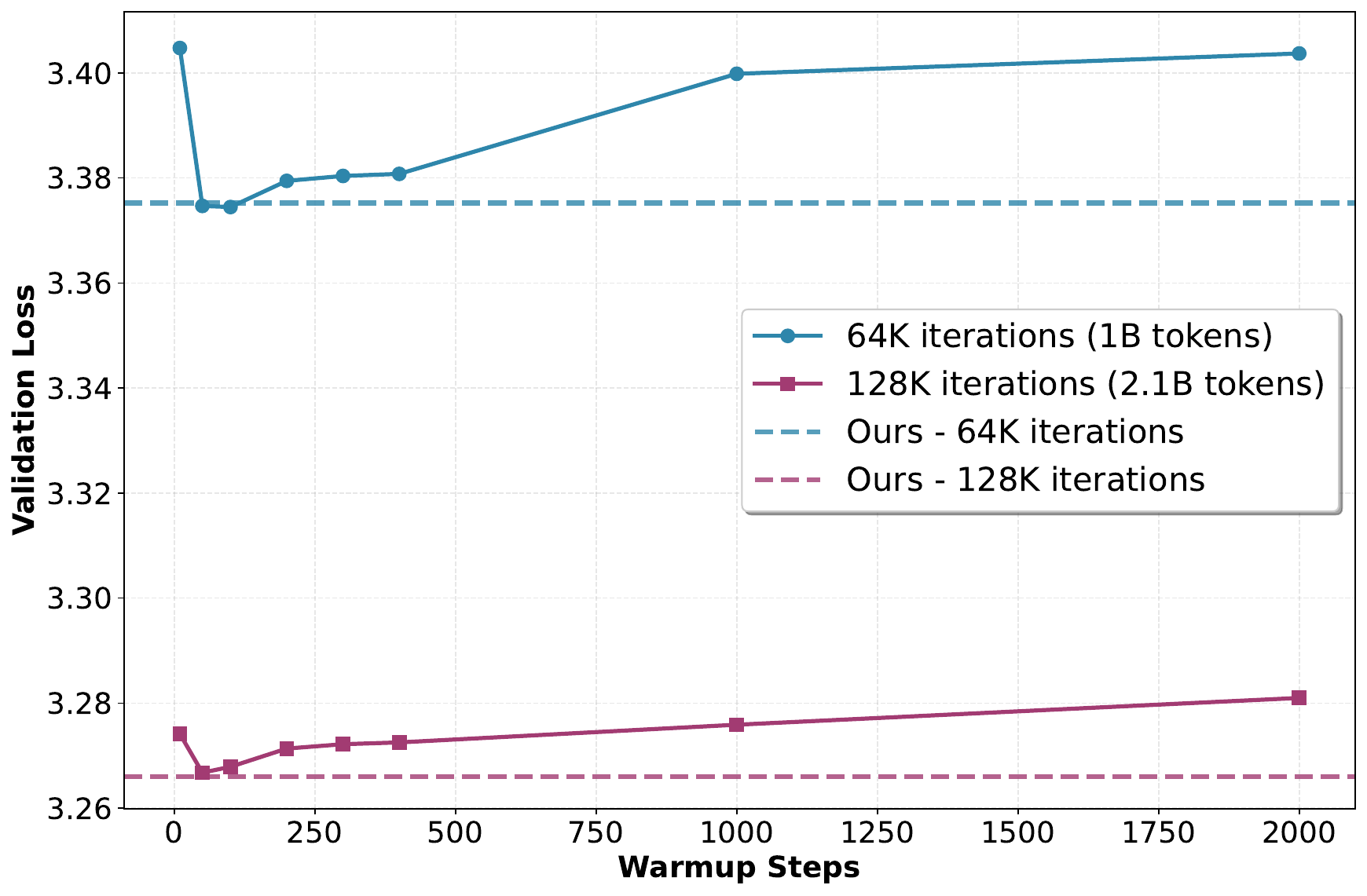}
    \end{subfigure}
    \hfill
    \begin{subfigure}{0.32\textwidth}
        \includegraphics[width=\linewidth]{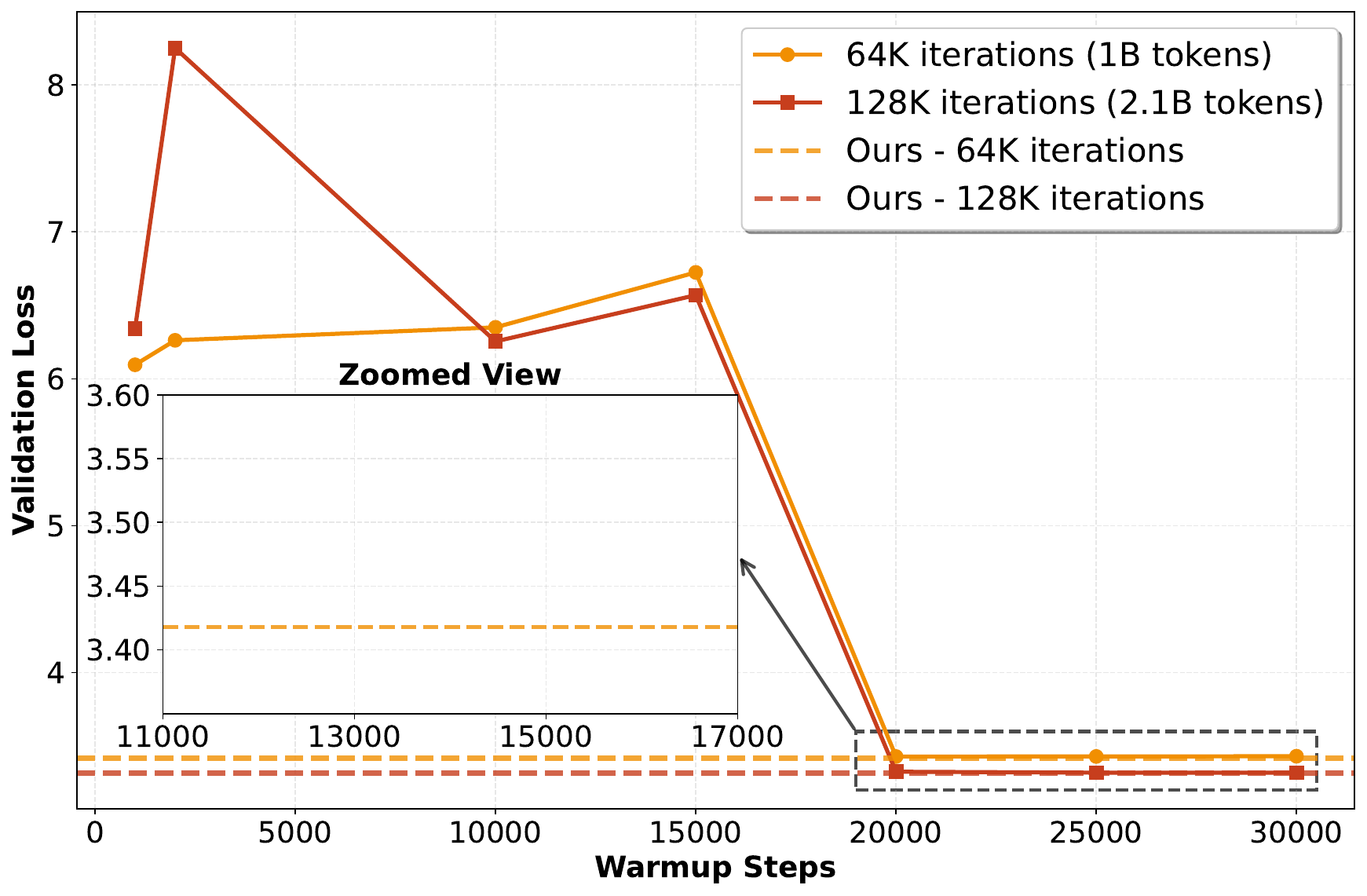}
    \end{subfigure}
    \hfill
    \begin{subfigure}{0.32\textwidth}
        \includegraphics[width=\linewidth]{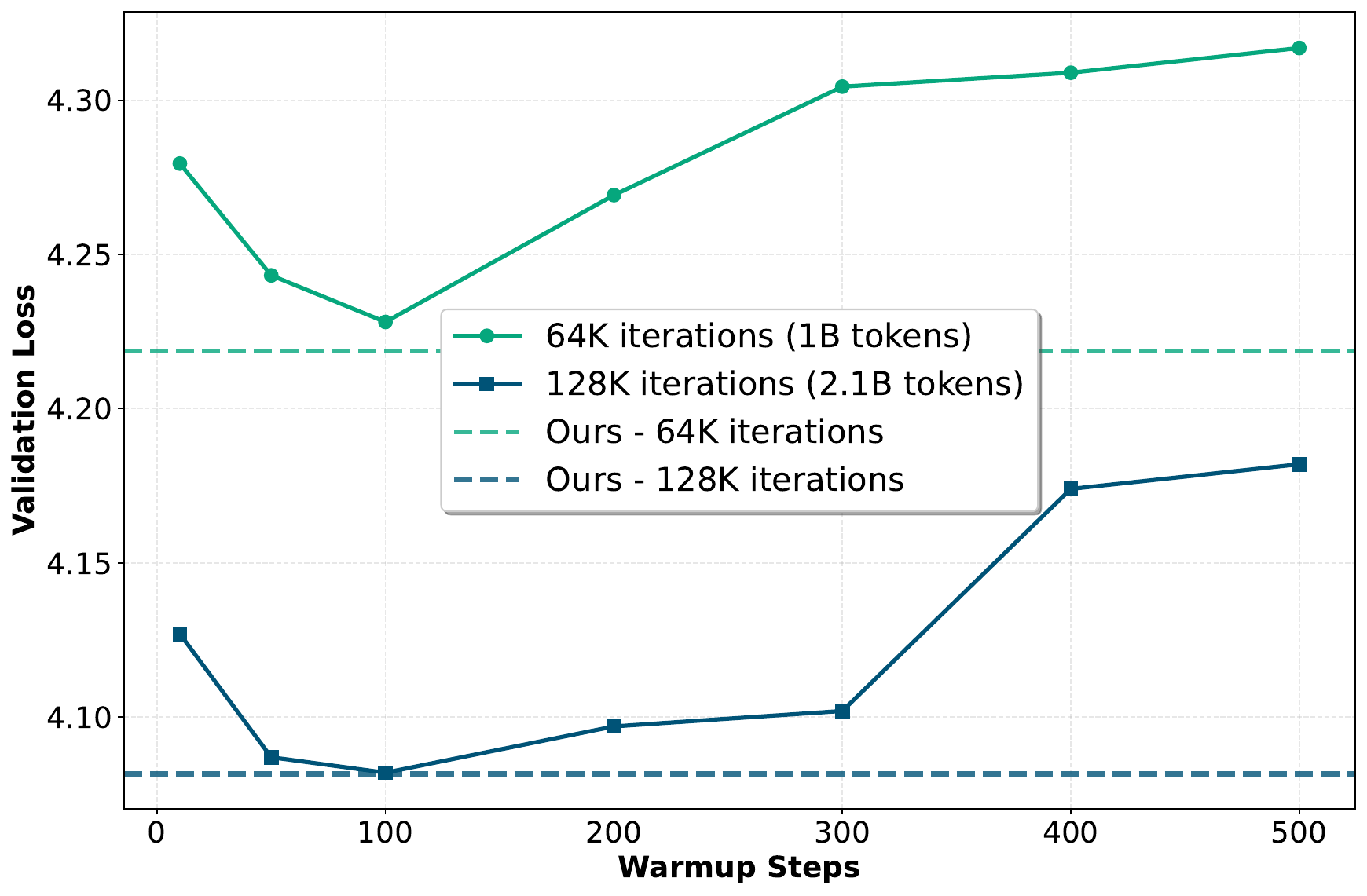}
    \end{subfigure}

    \begin{subfigure}{0.32\textwidth}
        \includegraphics[width=\linewidth]{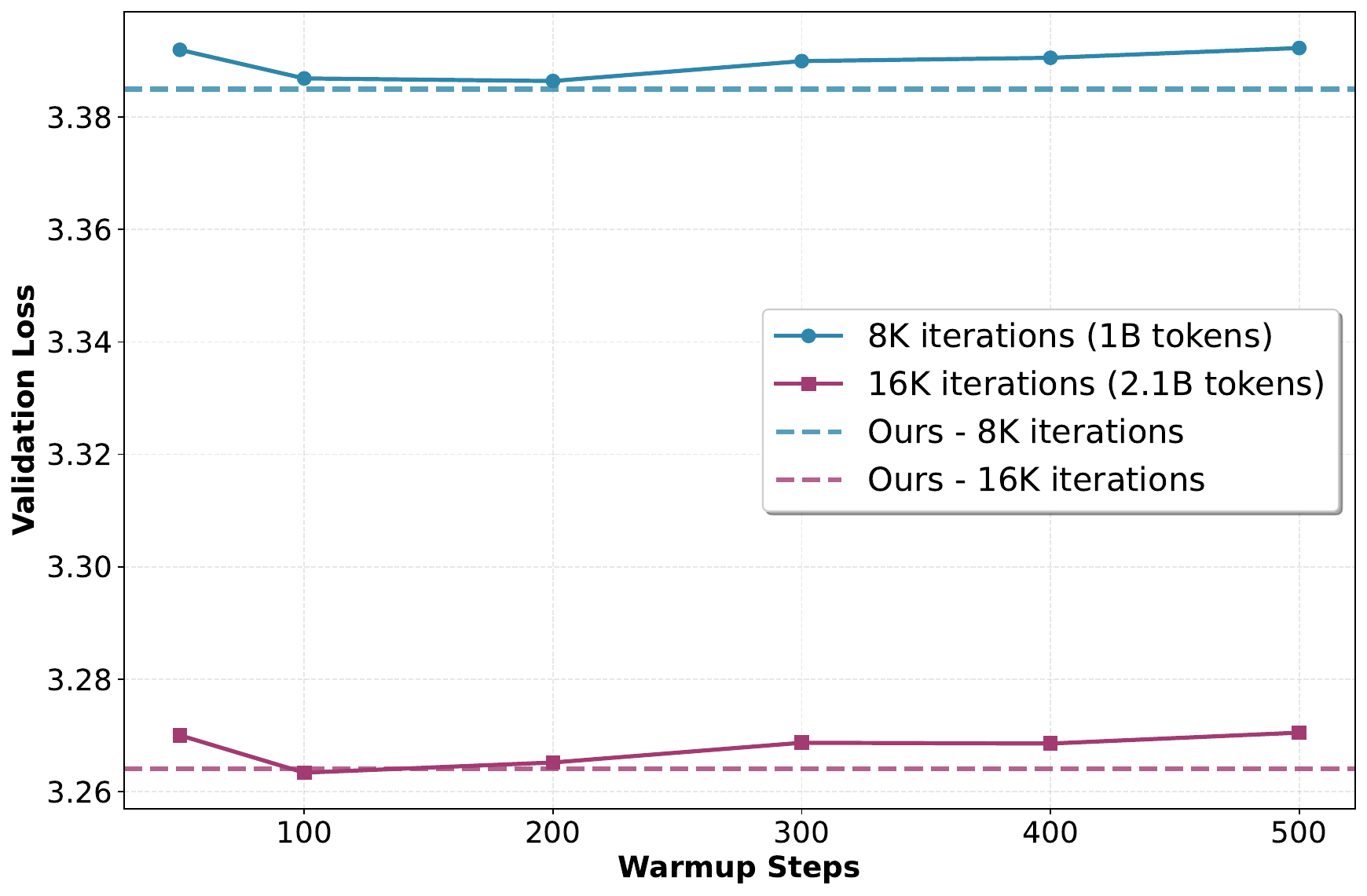}
        \caption{Muon}
    \end{subfigure}
    \hfill
    \begin{subfigure}{0.32\textwidth}
        \includegraphics[width=\linewidth]{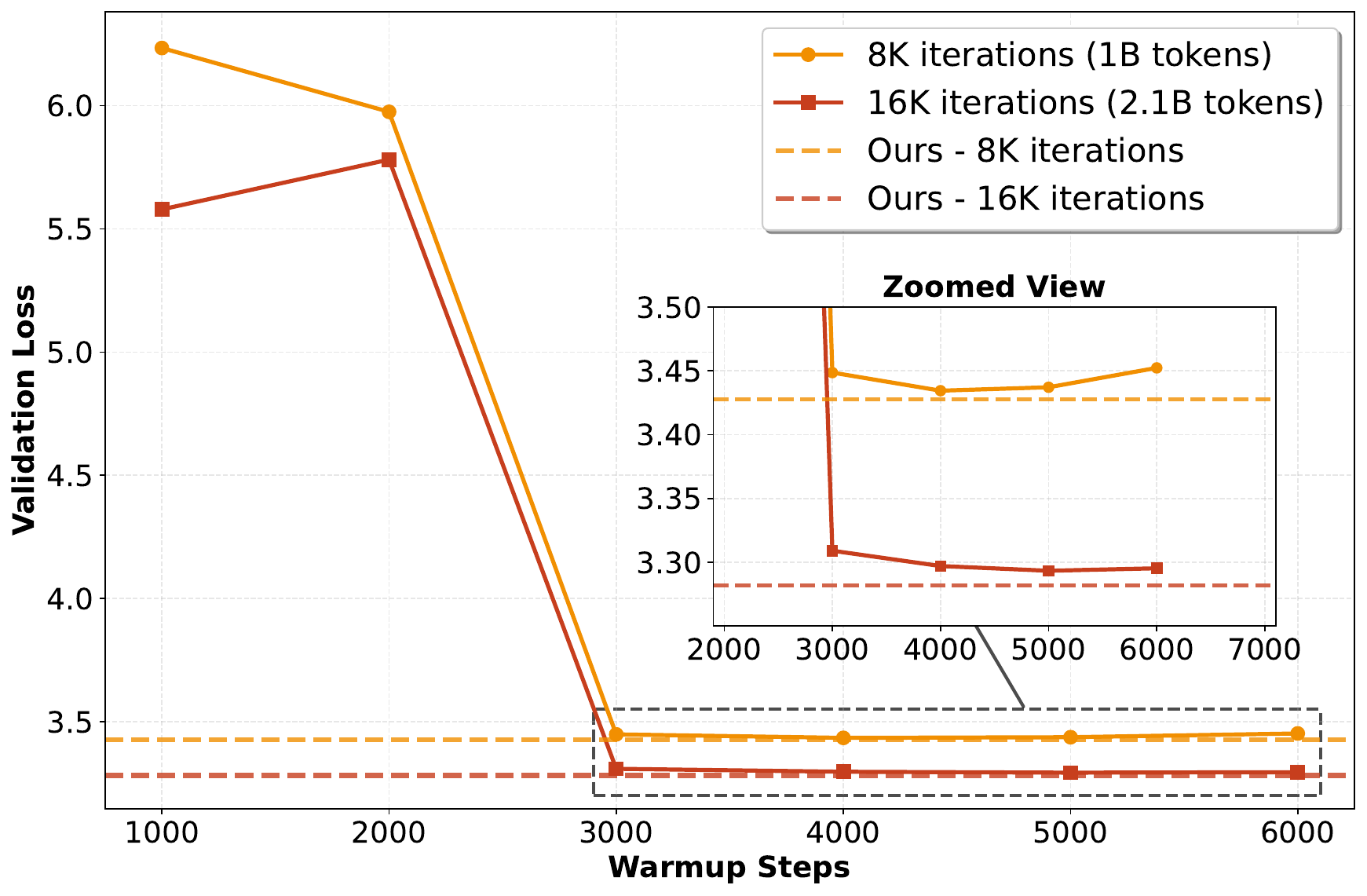}
        \caption{Lion}
    \end{subfigure}
    \hfill
    \begin{subfigure}{0.32\textwidth}
        \includegraphics[width=\linewidth]{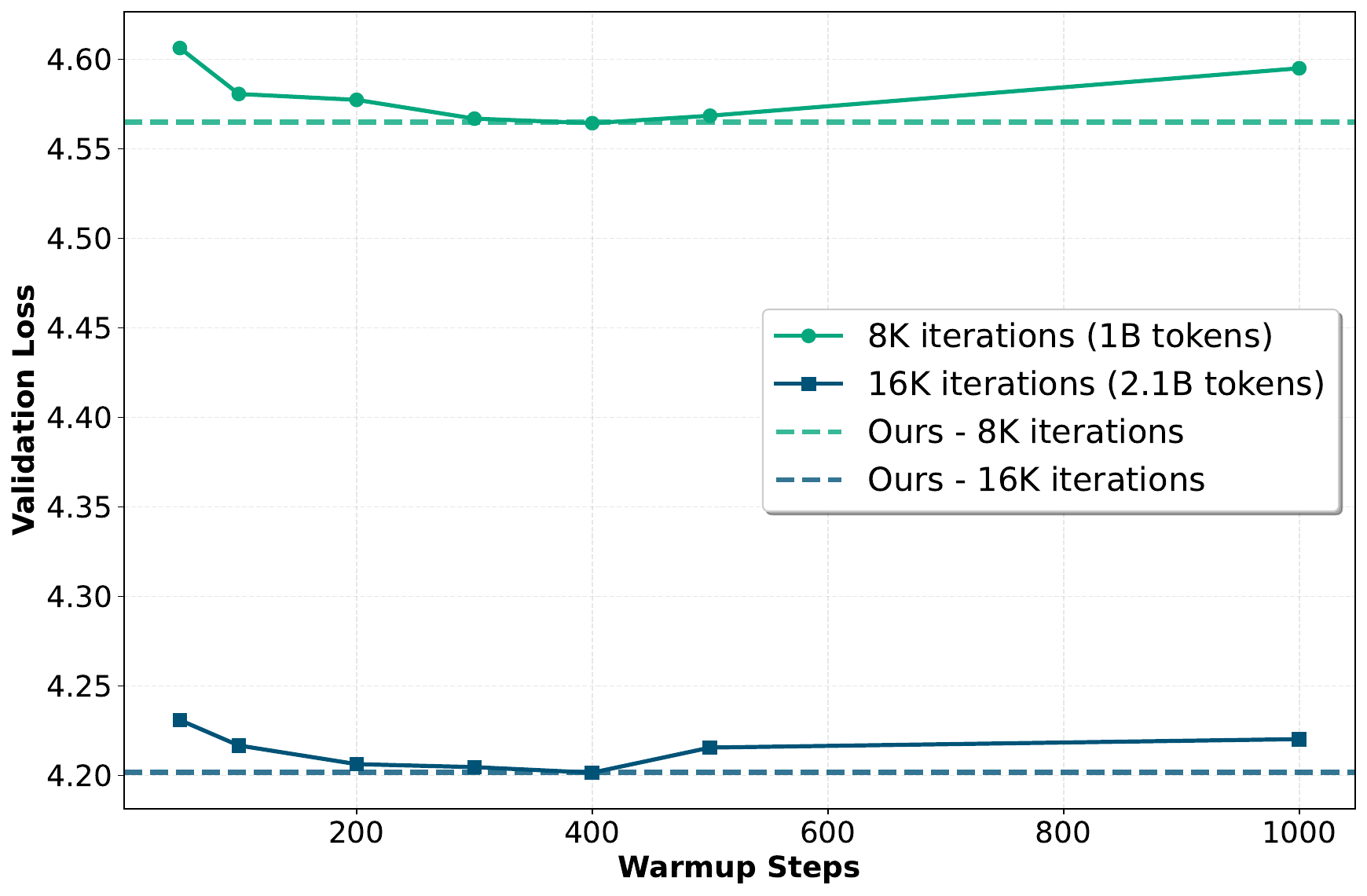}
        \caption{normSGD}
    \end{subfigure}
    \caption{Validation loss on LLaMA 210M model with $\mathrm{bs} = 32$ (top) and $256$ (bottom) as a function of manually selected warmup length (solid) vs.\ adaptive warmup (dashed). Across optimizers, the adaptive method outperforms or at least matches the best manually tuned value without search.}
    \label{fig:warmup-sweep-210M}
\end{figure*}

\subsection{Experiment Results}
\label{sec:llm_results}

\textbf{Language model pretraining.}
Across all optimizers, model sizes, and batch-sizes, the adaptive schedule improves or at least matches the best hand-tuned warm-up value, without requiring any grid search.

This effect is visible for Muon and normalized SGD, where the dashed line corresponding to our method lies at (or below) the minimum of the manual sweep. This suggests that rather than committing to a fixed linear ramp, allowing the learning rate to follow the rational shape $\eta(\Delta)$ yields a more stable transition into the high-curvature regime.

The Lion results highlight the sensitivity of the warm-up phase: choosing a warm-up length that is too short causes optimization to diverge.
Our adaptive approach navigates this regime robustly, automatically identifying an appropriate transition point $\Delta'$ and avoiding catastrophic failures.

The benefit of adaptive warm-up is most pronounced in the smaller batch-size regime ($\mathrm{bs}=32$), where gradient noise is higher and the optimization dynamics are more sensitive to aggressive early steps.

\begin{wrapfigure}[18]{r}{0.45\textwidth}
    \centering
    \includegraphics[width=0.45\textwidth]{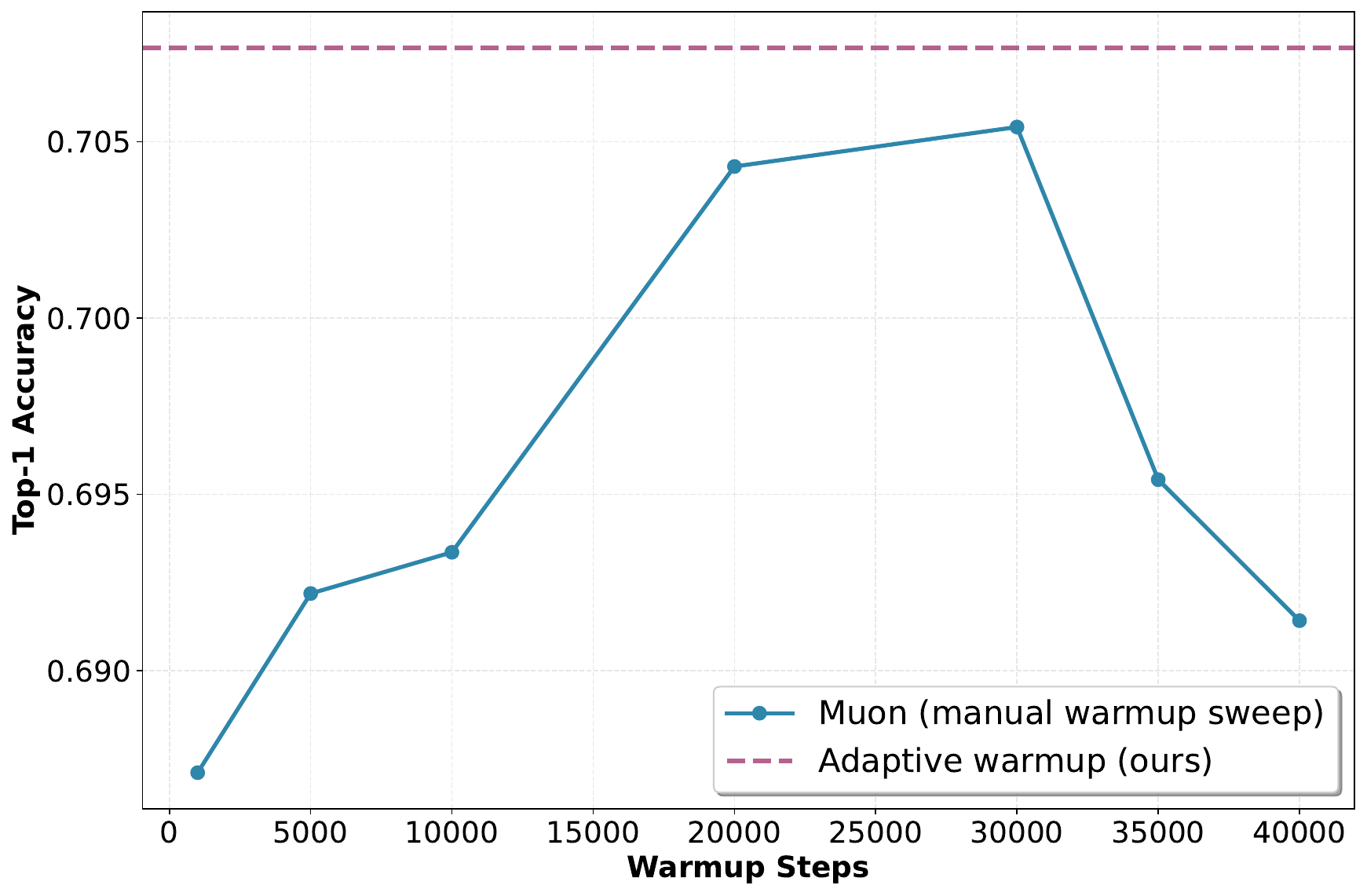}
    \caption{Accuracy on ImageNet-1K for Swin-B trained with Muon vs.\ manually selected warmup (solid). The dashed line shows the adaptive scheduler, which achieves the highest accuracy without any warmup search.}
    \label{fig:swin}
\end{wrapfigure}

In this setting, careful control of the warm-up phase is critical, and the adaptive scheduler provides a clear advantage over manually specified schedules.
For the larger batch-size regime ($\mathrm{bs}=256$), the performance gap between different warm-up choices narrows, consistent with more stable gradients. Nevertheless, the adaptive method remains robust and competitive.
Importantly, the qualitative trends are consistent across model scales: both the 124M and 210M models exhibit the same sensitivity patterns and benefit similarly from adaptive warm-up, indicating that the proposed scheduler generalizes across pretraining scales.

\textbf{Vision transformers.}
We evaluate a Swin Transformer Base (Swin-B)~\citep{liu2021swin} pretrained on ImageNet-1K~\citep{5206848} with Muon for 50K iterations, using $f^\star = 0.4$ as the target loss.
Figure~\ref{fig:swin} shows Top-1 accuracy vs.\ warmup length; the adaptive scheduler surpasses the best manually tuned result, confirming transfer to vision architectures without additional tuning.

\subsection{Ablations}
\label{sec:ablations}
\label{sec:f_star_abl}

\textbf{Target loss.}
The target loss $f^\star$ is the only parameter requiring prior task knowledge.
It enters via the suboptimality gap $\Delta^t = f(x^t) - f^\star$, which shifts the schedule along the $\Delta$ axis.
$f^\star$ does not need to be exact: it normalizes the loss scale, and any value within a reasonable range of the true optimum works well.
We study the sensitivity on LLaMA 124M pretraining with $\mathrm{bs}=32$ using Muon (Figure~\ref{fig:abl-fstar}).
Any $f^\star \in [3.0, 4.0]$ outperforms the best manually tuned warm-up; only extreme values ($f^\star = 0$ or $f^\star \geq 5$) cause noticeable degradation.
The best result is near $f^\star = 3.2$, matching the expected pretraining loss for this architecture on FineWeb.

\textbf{Frobenius-scale variance.}
We vary $\sigma_F^2$ over two orders of magnitude on the same setup (Figure~\ref{fig:abl-sigma}).
Values from $10^{-3}$ to $3 \times 10^{-3}$ all outperform the manually tuned baseline, with $\sigma_F^2 = 3 \times 10^{-3}$ achieving the lowest validation loss.
Very small values ($\sigma_F^2 \leq 10^{-4}$) degrade performance, as they effectively ignore gradient noise and cause the scheduler to underestimate the required warmup.
In all experiments we use $\sigma_F^2 = 10^{-3}$ without per-setup tuning.

\textbf{Schedule type.}
To assess whether the calibration overfits to a specific schedule family, we vary both the warmup and decay functional forms used as $\eta_{\mathrm{trgt}}$ in constraint~\eqref{eq:constr_3}.
We consider all combinations of linear and cosine warmup with linear and cosine decay, as well as WSD~\citep{hu2024minicpm} and WSO~\citep{yano2026wso} schedules (Figure~\ref{fig:abl-sched}).
Across all six schedule types the adaptive method consistently outperforms the manually tuned baseline, confirming that the improvement is not an artifact of matching a particular functional form.

Taken together, these three ablations show that the scheduler is robust to the choice of $f^\star$, $\sigma_F^2$, and the schedule family: broad neighborhoods of any reasonable configuration all outperform the best manually tuned warm-up, with no additional search required.

\begin{figure*}[h!]
    \centering
    \begin{subfigure}{0.32\textwidth}
        \includegraphics[width=\linewidth]{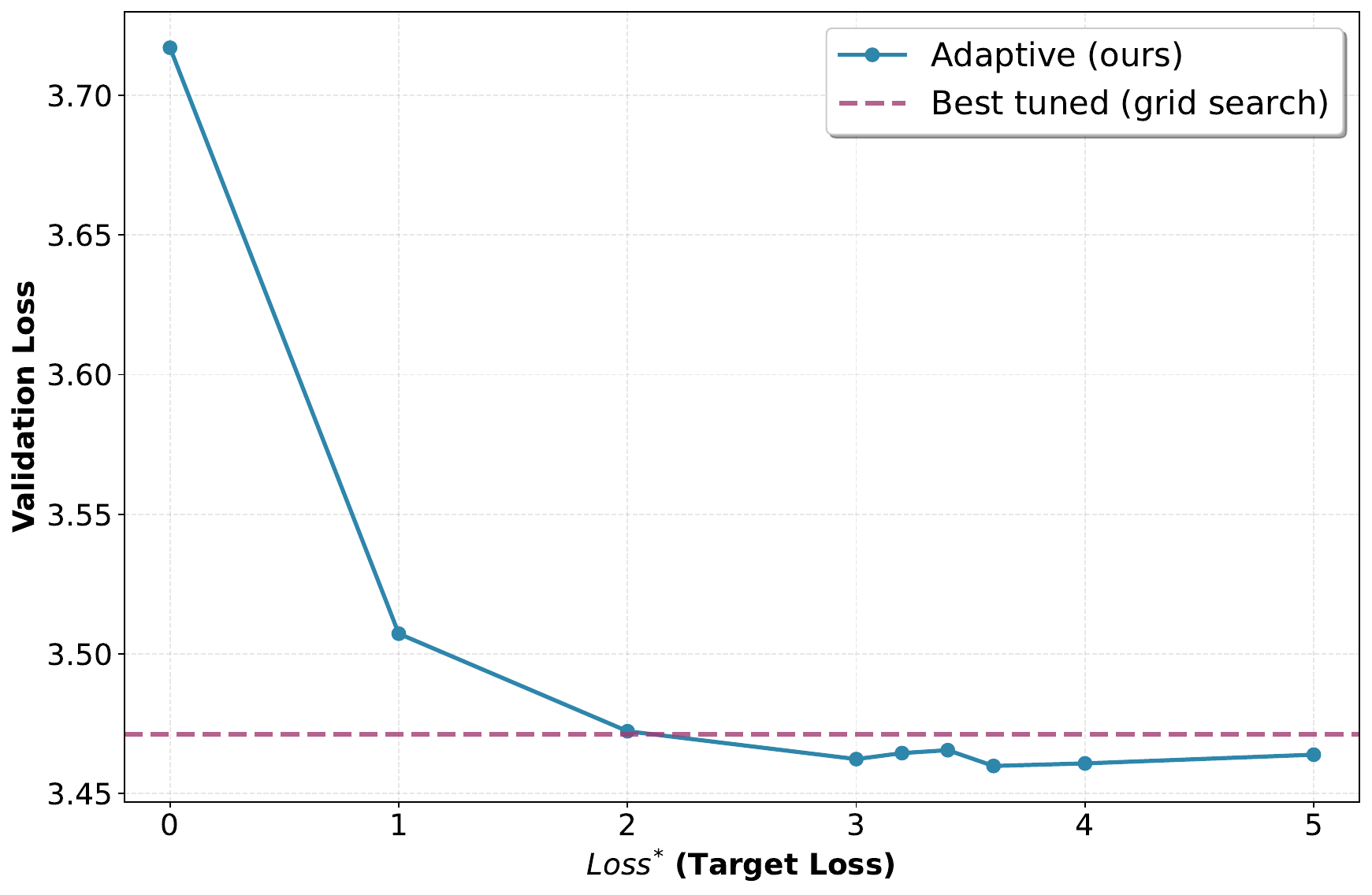}
        \caption{Target loss $f^\star$}
        \label{fig:abl-fstar}
    \end{subfigure}
    \hfill
    \begin{subfigure}{0.32\textwidth}
        \includegraphics[width=\linewidth]{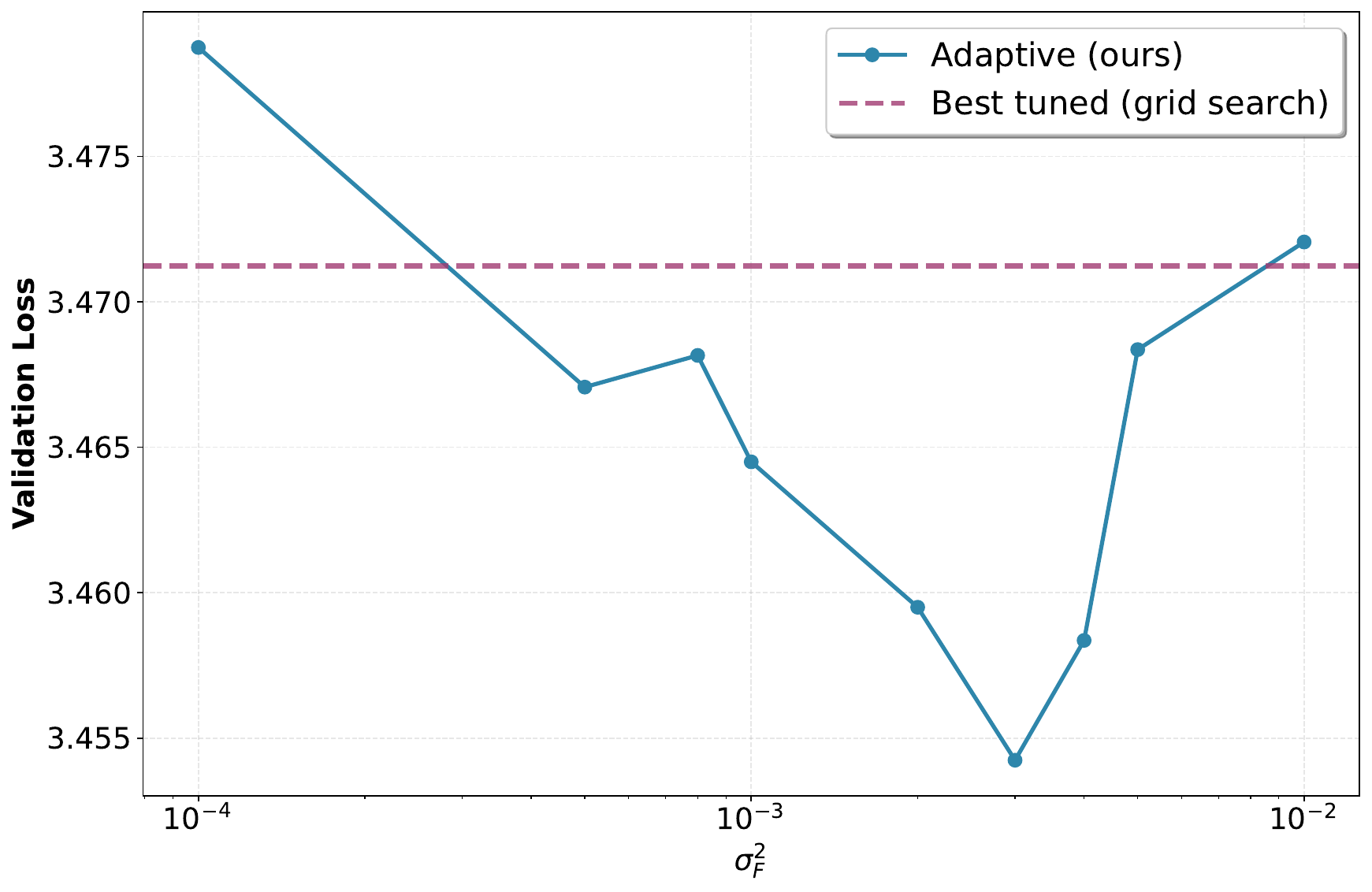}
        \caption{Frobenius variance $\sigma_F^2$}
        \label{fig:abl-sigma}
    \end{subfigure}
    \hfill
    \begin{subfigure}{0.32\textwidth}
        \includegraphics[width=\linewidth]{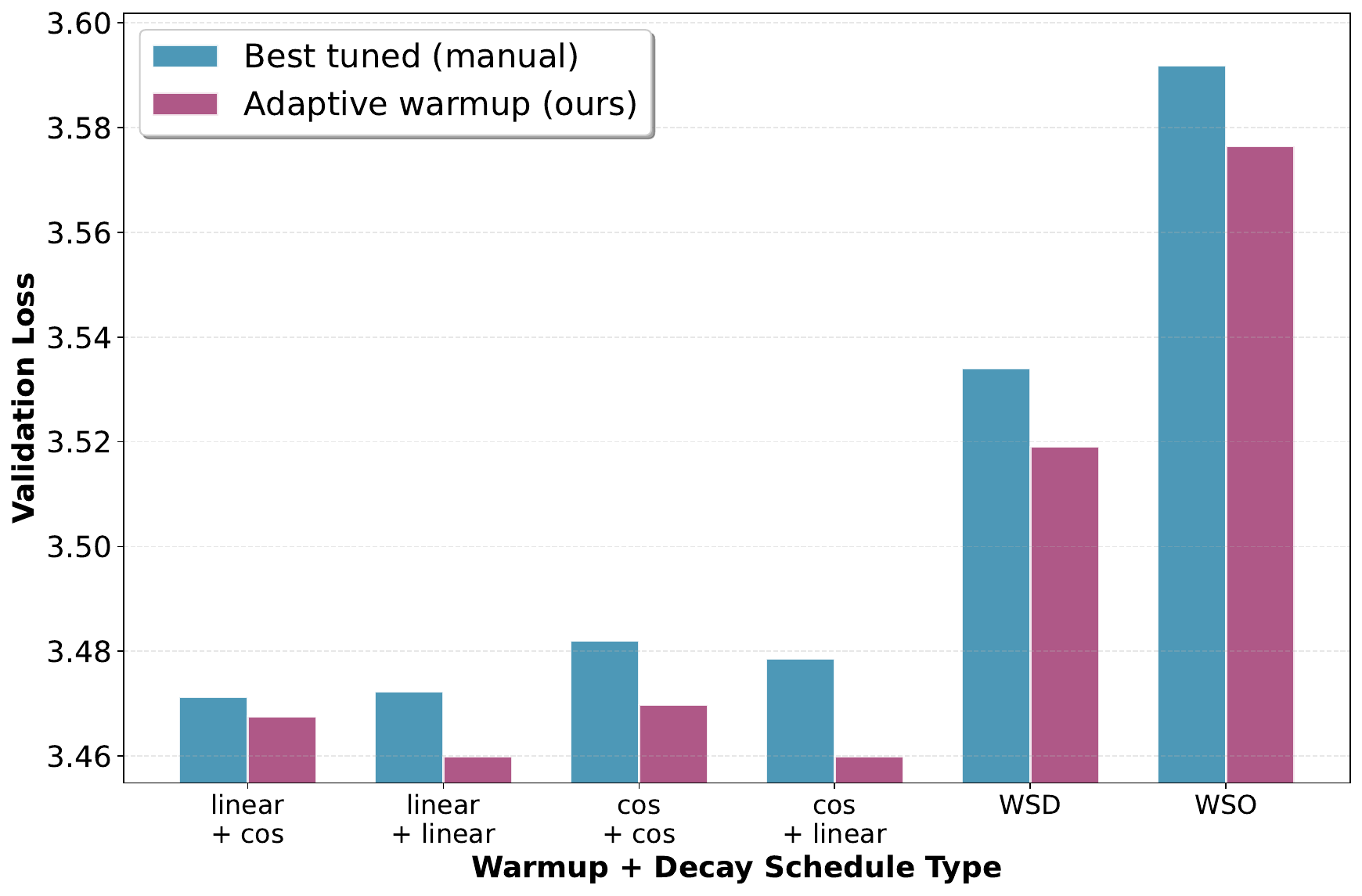}
        \caption{Schedule type $\eta_{\mathrm{trgt}}$}
        \label{fig:abl-sched}
    \end{subfigure}
    \caption{Ablation studies on LLaMA 124M pretraining with $\mathrm{bs}=32$ using Muon. The dashed line in (a) and (b) shows the best manually tuned warm-up; in (c), blue bars show the best manual warmup and red bars show the adaptive scheduler.}
    \label{fig:ablations}
\end{figure*}

\section{Conclusion}

We studied learning rate scheduling for LMO-based optimizers under a suboptimality-dependent smoothness assumption. Under this assumption, warm-up followed by decay emerges from the convergence proof rather than being imposed heuristically, with the theoretical schedule taking the rational form $\eta(\Delta) = \Delta / \mathcal{K}(\Delta)$.
We then derived a practical scheduler that determines the warm-up duration automatically from standard hyperparameters, with $\kappa$ computed analytically and a single fixed $\sigma_F^2$ across all setups.
Experiments on LLaMA 124M and 210M pretraining with Muon, Lion, and normSGD, and on Swin-B ImageNet training, show that the adaptive warm-up matches or exceeds the best manually tuned schedule in every configuration tested.

\newpage



\end{mainpart}

\newpage

\tableofcontents

\begin{appendixpart}

\section{Automatic Boundedness for the Euclidean Norm}
\label{app:boundedness}

The following lemma confirms that for the Euclidean norm, Assumption~\ref{ass:boundedness} is automatically satisfied by our learning rate schedule.

\begin{lemma}
\label{lem:boundedness}
Suppose Assumptions~\ref{ass:star_convexity} and~\ref{ass:smoothness2} hold under the Euclidean norm $\|\cdot\| = \|\cdot\|_2$. Consider the iterates $x^t$ generated by \eqref{eq:update_rule_1} with $g^t = \nabla f(x^t)$ and learning rate $\eta^t = \frac{\Delta^t}{D \cdot \mathcal{K}(x^t)}$, where $D := \|x^0 - x^\star\|_2$.
Then $\|x^t - x^\star\|_2 \le D$ for all $t \ge 0$.
\end{lemma}

\begin{proof}
We proceed by induction on $t$. The base case $t=0$ holds trivially since $\|x^0 - x^\star\|_2 = D$.

For the inductive step, assume $\|x^t - x^\star\|_2 \le D$. From the update rule \eqref{eq:update_rule_1} with the Euclidean norm:
\[
x^{t+1} = x^t - \eta^t \frac{\nabla f(x^t)}{\|\nabla f(x^t)\|_2}.
\]
Using the star-convexity condition (Assumption~\ref{ass:star_convexity}):
\[
\langle \nabla f(x^t), x^t - x^\star \rangle \ge f(x^t) - f^\star = \Delta^t > 0.
\]
We compute:
\begin{align*}
\|x^{t+1} - x^\star\|_2^2 &= \|x^t - x^\star\|_2^2 - 2\eta^t \frac{\langle \nabla f(x^t), x^t - x^\star \rangle}{\|\nabla f(x^t)\|_2} + (\eta^t)^2 \\
&\le \|x^t - x^\star\|_2^2 - 2\eta^t \frac{\Delta^t}{\|\nabla f(x^t)\|_2} + (\eta^t)^2.
\end{align*}
By Assumption~\ref{ass:smoothness2}, $\|\nabla f(x^t)\|_2 \le \mathcal{K}(x^t) \|x^t - x^\star\|_2 \le \mathcal{K}(x^t) D$.
With $\eta^t = \frac{\Delta^t}{D \cdot \mathcal{K}(x^t)}$:
\begin{align*}
\|x^{t+1} - x^\star\|_2^2 &\le D^2 - 2 \cdot \frac{\Delta^t}{D \cdot \mathcal{K}(x^t)} \cdot \frac{\Delta^t}{\mathcal{K}(x^t) D} + \frac{(\Delta^t)^2}{D^2 \mathcal{K}(x^t)^2} \\
&= D^2 - \frac{2(\Delta^t)^2}{D^2 \mathcal{K}(x^t)^2} + \frac{(\Delta^t)^2}{D^2 \mathcal{K}(x^t)^2} \\
&= D^2 - \frac{(\Delta^t)^2}{D^2 \mathcal{K}(x^t)^2} \le D^2.
\end{align*}
Thus $\|x^{t+1} - x^\star\|_2 \le D$, completing the induction.
\end{proof}

\section{Empirical Motivation for Assumption~\ref{ass:smoothness2}}
\label{app:smooth}

To complement the main-text visualization based on Lion (Figure \ref{fig:smoothness-lion}), we report the same Lipschitz–style diagnostic for Muon and normalized SGD. For each optimizer, we plot the empirical ratio
\[
\mathcal{K}^t \;=\;
\frac{\|\nabla f(x^{t+1}) - \nabla f(x^{t})\|_\star}{\|x^{t+1}-x^{t}\|}
\qquad\text{versus}\qquad
\Delta^t = f(x^t)-f^\star,
\]
and overlay a fitted curve of the form
\[
K_0 + K_1 \Delta + K_2 \Delta^2.
\]

The resulting trajectories, shown in Figures~\ref{fig:smoothness-lion} and \ref{fig:lips-muon-norm}, reveal three consistent phenomena.

\begin{enumerate}
\item \textbf{Suboptimality-dependent curvature is essential.}
Across all optimizers, the empirical dependence of $\mathcal{K}^t$ on $\Delta_t$ exhibits clear curvature, and the quadratic term dominates the linear component except near the beginning of training. This confirms that $K_\rho$ must be strictly positive in Assumption~\ref{ass:smoothness2}.

\item \textbf{Different optimizers induce different curvature profiles.}
The magnitude and slope of $\mathcal{K}^t$ vary substantially across methods. Muon produces smoother trajectories with a tight parabolic fall-off, while normSGD displays a steeper and higher-variance decrease as training progresses. These differences reflect the underlying geometry of each update rule—LMO direction choice, normalization scale, and implicit conditioning—and highlight why a single global Lipschitz constant would poorly approximate training dynamics.

\item \textbf{The trend is optimizer-universal.}
Even though the shapes differ, both optimizers exhibit the same qualitative pattern as Lion in the main text: the curvature decreases monotonically with $\Delta_t$ and cannot be explained by a constant Lipschitz term alone. That is, \emph{all three} optimizers behave as predicted by Assumption~\ref{ass:smoothness2}, but with optimizer-dependent constants.
\end{enumerate}

Overall, these additional results reinforce that $K_\rho>0$ is not an analytical convenience but an empirically grounded property of modern LMO optimizers, and that the precise smoothness structure is optimizer-specific—supporting our choice to estimate $K_0,K_1,K_2$ rather than assume a fixed curvature model.

\begin{figure*}[h]
    \centering
    \begin{subfigure}{0.48\textwidth}
        \includegraphics[width=\linewidth]{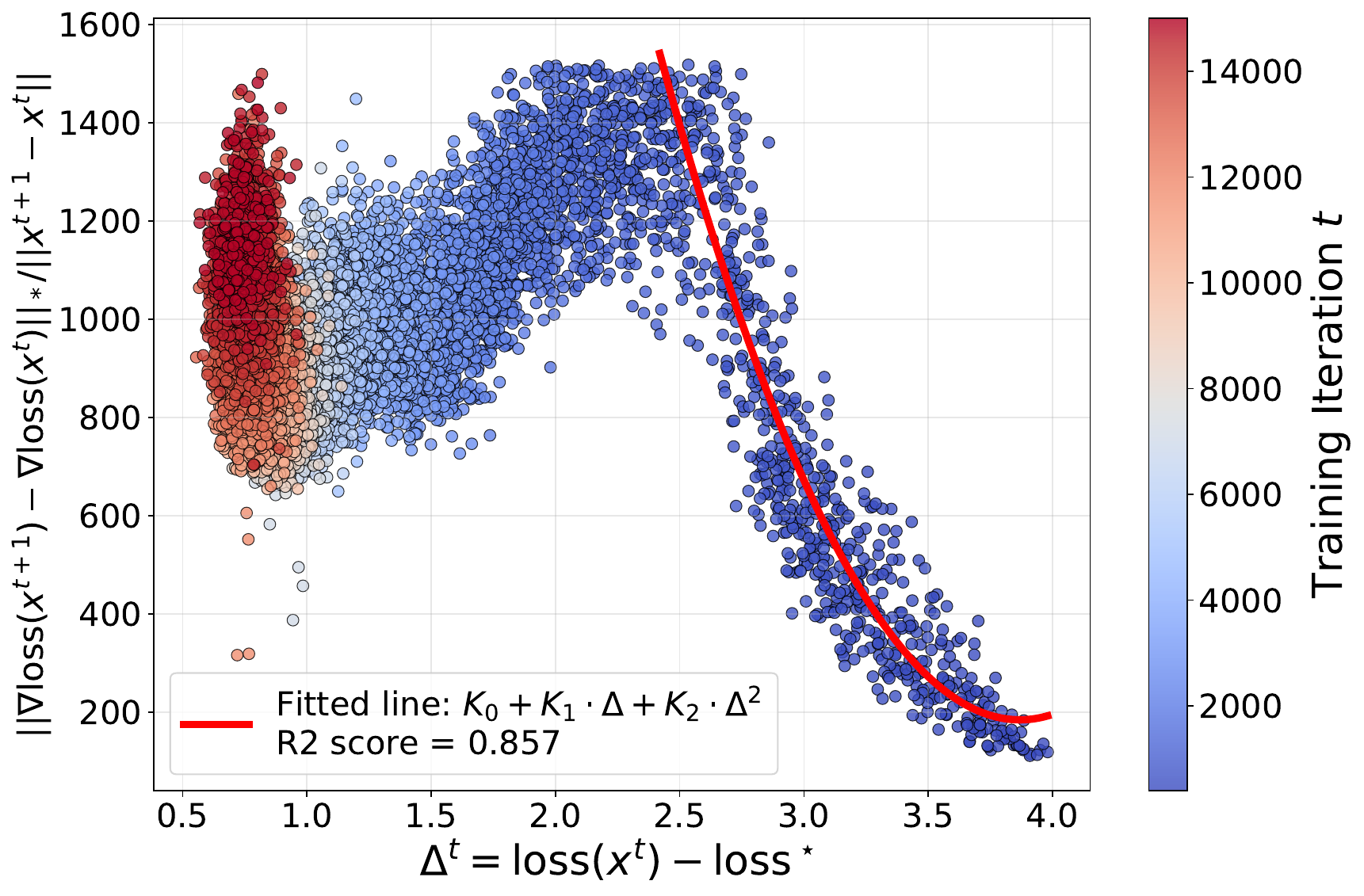}
        \caption{Muon}
    \end{subfigure}
    \hfill
    \begin{subfigure}{0.48\textwidth}
        \includegraphics[width=\linewidth]{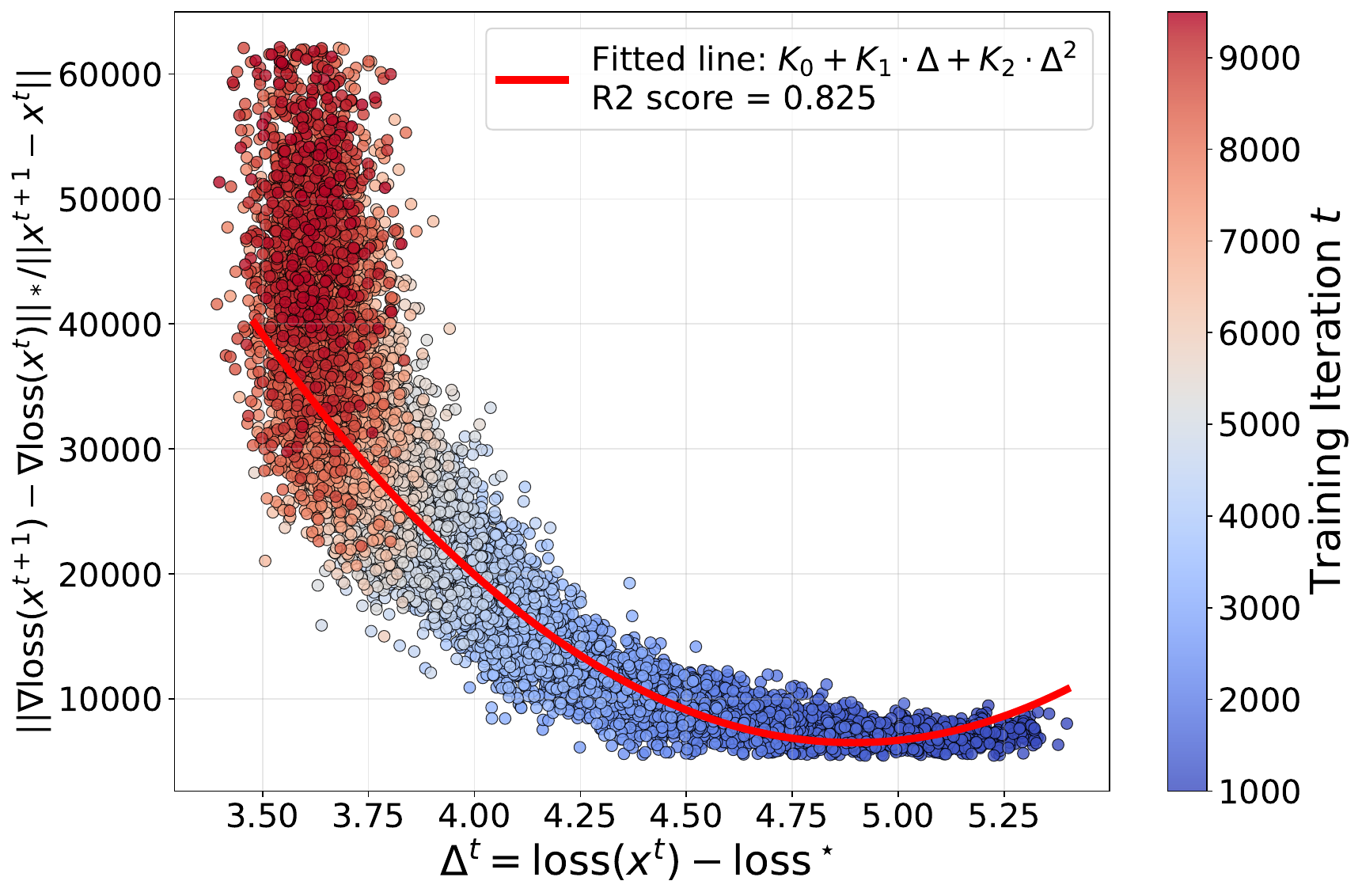}
        \caption{normSGD}
    \end{subfigure}
    \caption{Empirical smoothness ratio $\mathcal{K}^t$ versus suboptimality gap $\Delta_t$ for Muon and normalized SGD. The fitted curve $K^0 + K^1\Delta + K^2\Delta^2$ (red) demonstrates that a quadratic term is necessary to explain the observed trajectory, and that the magnitude and curvature differ across optimizers.}
    \label{fig:lips-muon-norm}
\end{figure*}

\textbf{Experimental setup.}
All measurements are obtained using the same model, data pipeline, and optimizer hyperparameters as in Section~\ref{sec:exp}, with settings inherited from~\cite{semenov2025benchmarking}. The only aspect that varies across optimizers is the number of optimization steps observed. For each iteration, we estimate $\Delta^t = f_{\xi^t}(x^t)-f^\star$ and compute
\[
\mathcal{K}^t_{\xi^t} \;=\; \frac{\|\nabla f_{\xi^t}(x^{t+1}) - \nabla f_{\xi^t}(x^t)\|_\star}{\|x^{t+1}-x^t\|}
\]
on the \emph{same} mini-batch $\xi^t$ of size $64$ by performing an additional forward and backward pass, ensuring that the difference is not corrupted by batch-to-batch noise. While both $\Delta^t$ and $\mathcal{K}^t_{\xi^t}$ remain stochastic due to minibatch sampling, the batch is sufficiently large for these quantities to provide faithful estimates of the smoothness trend under the same training regime used in the main experiments.

\subsection{Relation to Edge of Stability}
\label{app:eos}

The Edge of Stability (EoS) phenomenon~\citep{cohen2021gradient, damian2023selfstabilization} refers to the empirical observation that, during full-batch gradient descent, the top Hessian eigenvalue $\lambda_{\max}(\nabla^2 f(x^t))$ rises to approximately $2/\eta$ and then stabilizes.
At first glance this appears to contradict Assumption~\ref{ass:smoothness2}, under which the local smoothness $\mathcal{K}(x^t) = K_0 + K_1\Delta^t + K_\rho(\Delta^t)^\rho$ decreases as training proceeds.
Three observations resolve this.

First, $\mathcal{K}^t$ and $\lambda_{\max}(\nabla^2 f(x^t))$ measure different things.
The top eigenvalue is the worst-case curvature over all directions.
The ratio $\mathcal{K}^t = \|\nabla f(x^{t+1}) - \nabla f(x^t)\|_\star / \|x^{t+1} - x^t\|$ is the gradient change specifically along the LMO update direction $x^{t+1} - x^t$, which is constrained to the unit ball of the optimization norm.
For LMO methods, this direction avoids the high-curvature subspace where $\lambda_{\max}$ is attained.
The two quantities can move in opposite directions simultaneously.

Second, EoS has been studied exclusively for full-batch gradient descent~\citep{cohen2021gradient, damian2023selfstabilization}.
LMO methods such as Muon apply Newton-Schulz orthogonalization, projecting the gradient onto the spectral norm ball~\citep{jordan2024muon, chen2026muon}.
This fundamentally changes the update direction relative to gradient descent, and no EoS analog has been established for such methods.

Third, Assumption~\ref{ass:smoothness2} is a global technical condition required for the convergence proof.
Our empirical validation in Figures~\ref{fig:smoothness-lion} and \ref{fig:lips-muon-norm} measures $\mathcal{K}^t$ along the actual LMO optimization trajectory, where the assumption is relevant.
Global worst-case sharpness (over all directions) can grow while directional smoothness along LMO steps decreases, and this is precisely what EoS and Assumption~\ref{ass:smoothness2} each capture in their respective domains.

\section{Connection to Specific Optimizers: Detailed Derivations}
\label{app:optimizers}

This general LMO-based framework in \eqref{eq:update_rule_1} unifies several popular optimizers by instantiating different norms $\|\cdot\|$ on the parameter space. Table~\ref{tab:optimizers} summarizes the correspondence.

\begin{table}[h]
\centering
\begin{tabular}{@{}llll@{}}
\toprule
\textbf{Optimizer} & Norm & Dual & $\mathrm{LMO}(g)$ \\
\midrule
normSGD & $\ell_2$ & $\ell_2$ & $-g/\|g\|$ \\
signSGD/Lion & $\ell_{\infty}$ & $\ell_1$ & $-\mathrm{sign}(g)$ \\
Muon & Spectral & Nuclear & $-UV^{\top}$ \\
Layer-wise & $\max_i\|{\cdot}\|_{(i)}$ & $\sum_i\|{\cdot}\|_{(i),\star}$ & per-layer \\
\bottomrule
\end{tabular}
\caption{Norm choice and corresponding optimizers. The update rule \eqref{eq:update_rule_1} recovers each optimizer.}
\label{tab:optimizers}
\end{table}

Recall that the linear minimization oracle is
\[
\operatorname{LMO}(g) = \arg\min_{\|t\|=1} \langle g, t \rangle = -\frac{g^{\star}}{\|g^{\star}\|},
\]
where $g^{\star}$ is the subgradient of the dual norm $\|g\|_{\star} = \max_{\|t\|=1} \langle g, t \rangle$. Below we derive the LMO for each norm choice.

\paragraph{Normalized Gradient Descent ($\ell_2$ norm).}
For the Euclidean norm $\|\cdot\| = \|\cdot\|_2$, the dual norm is also $\ell_2$ (self-duality), and $\operatorname{LMO}(g) = -g/\|g\|_2$. The update \eqref{eq:update_rule_1} becomes:
\[
x^{t+1} = x^t - \eta^t \frac{\nabla f(x^t)}{\|\nabla f(x^t)\|_2},
\]
which is exactly \textbf{Normalized Gradient Descent} \citep{nesterov1984minimization,hazan2015beyond,levy2016power}. This method decouples the direction from the magnitude, ensuring unit-norm steps.

\paragraph{signSGD ($\ell_{\infty}$ norm).}
For the $\ell_{\infty}$ norm $\|t\|_{\infty} = \max_i |t_i|$, the dual norm is $\ell_1$: $\|g\|_{\star} = \|g\|_1 = \sum_i |g_i|$. The LMO has the closed form:
\[
\operatorname{LMO}(g) = -\operatorname{sign}(g),
\]
where $\operatorname{sign}(g)_i = \operatorname{sign}(g_i) \in \{-1, 0, +1\}$. The update \eqref{eq:update_rule_1} becomes:
\[
x^{t+1} = x^t - \eta^t \operatorname{sign}(\nabla f(x^t)),
\]
which is exactly \textbf{signSGD} \citep{bernstein2018signsgd}. This connection was noted by \citet{bernstein2024old}, who observed that signSGD is steepest descent under the $\ell_{\infty}$ norm.

\paragraph{Lion ($\ell_{\infty}$ norm with momentum).}
Lion~\citep{chen2023lion} extends signSGD by incorporating momentum:
\[
m^{t+1} = \beta_2 m^t + (1-\beta_2) \nabla f(x^t), \quad x^{t+1} = x^t - \eta^t \operatorname{sign}(\beta_1 m^t + (1-\beta_1) \nabla f(x^t)).
\]
Since the update direction is still given by the sign function, Lion corresponds to the $\ell_{\infty}$ norm geometry with a momentum-averaged gradient. This optimizer was discovered via symbolic search and has demonstrated strong performance across vision, language, and diffusion models while being more memory-efficient than Adam.

\paragraph{Muon (spectral norm on matrices).}
For matrix-valued parameters $W \in \mathbb{R}^{m \times n}$, the spectral norm is $\|W\|_{\text{op}} = \sigma_{\max}(W)$ (largest singular value). Its dual is the nuclear norm $\|G\|_{\star} = \sum_i \sigma_i(G)$ (sum of singular values). The LMO is the orthogonalized gradient:
\[
\operatorname{LMO}(G) = -UV^{\top},
\]
where $G = U\Sigma V^{\top}$ is the SVD. The update \eqref{eq:update_rule_1} becomes:
\[
W^{t+1} = W^t - \eta^t \operatorname{Ortho}(\nabla f(W^t)),
\]
where $\operatorname{Ortho}(G) = UV^{\top}$ is the nearest (semi-)orthogonal matrix to $G$. This is exactly the \textbf{Muon} optimizer \citep{jordan2024muon}, which uses Newton-Schulz iterations to efficiently approximate the orthogonalization. Recent work has scaled Muon to large language models \citep{liu2025muonscalablellmtraining}.

\paragraph{Layer-wise optimizers (supremum norm over layers).}
In practice, modern neural networks consist of $L$ layers with parameters $(W_1, \ldots, W_L)$, where each layer may have different dimensions and structure. A natural approach is to apply different norms to different layers and combine them via the supremum:
\[
\|(W_1, \ldots, W_L)\| := \max_{i=1,\ldots,L} \|W_i\|_{(i)},
\]
where $\|\cdot\|_{(i)}$ denotes the norm chosen for the $i$-th layer. The dual norm is then
\[
\|(G_1, \ldots, G_L)\|_{\star} = \sum_{i=1}^{L} \|G_i\|_{(i),\star},
\]
and the LMO decomposes layer-wise:
\[
\operatorname{LMO}(G_1, \ldots, G_L) = \left( \operatorname{LMO}_{(i)}(G_i) \right)_{i=1}^{L},
\]
where $\operatorname{LMO}_{(i)}$ is the LMO corresponding to the norm $\|\cdot\|_{(i)}$.

This framework allows combining different optimizers for different layers---for example, using Muon (spectral norm) for large attention matrices, signSGD/Lion ($\ell_\infty$ norm) for embedding layers, and Normalized GD ($\ell_2$ norm) for small dense layers. Such hybrid strategies are commonly used in practice~\citep{jordan2024muon,liu2025muonscalablellmtraining}.

\section{Details of the Adaptive Scheduler from Section~\ref{sec:our_sceduler}}

\subsection{Closed-form expressions for \texorpdfstring{$\kappa$}{kappa}}
\label{app:kappas}

Recall that our learning-rate matching objective uses a Gaussian weight
\[
\tilde{\Delta} \sim \mathcal{N} \bigl(\Delta';\sigma_F^2 / \kappa\bigr),
\]
where $\kappa$ controls the scale of the effective step in Frobenius norm. Formally,
\[
\kappa \;:=\;
\sup_{\|u\| = 1} \bigl\{\|u\|_F^2\bigr\},
\]
where $\|\cdot\|$ is the norm used inside the LMO update and $\|\cdot\|_F$ is the Frobenius norm. In our implementation, parameters are grouped by layers, and the global optimization norm is defined as the maximum over layers of the per-layer norm. Consider a model consisting of layers indexed by $\ell = 1,\dots,L$, where the $\ell$-th weight matrix has dimensions $m_\ell \times n_\ell$.  
Under the max-over-layers geometry used throughout, the closed-form expressions for $\kappa$ are:

\begin{itemize}
\item \textbf{Muon (spectral norm).}  
Each layer uses the spectral norm $\|\cdot\|_{\mathrm{op}}$ and its dual, the nuclear norm $\|\cdot\|_*$.  
For matrices with operator norm bounded by~$1$,
\[
\sup_{\|A\|_{\mathrm{op}}=1} \|A\|_F^2 = \mathrm{rank}(A) = \min(m_\ell,n_\ell),
\]
and the nuclear norm satisfies $\|A\|_* \le \mathrm{rank}(A) \|A\|_{\mathrm{op}}$.  
With a max-over-layers constraint, the total worst-case Frobenius contribution sums across layers:
\[
\kappa_{\mathrm{Muon}} = \sum_{\ell=1}^{L} \min(m_\ell,n_\ell).
\]

\item \textbf{Lion / Sign-type (entrywise $\ell_\infty$ norm).}  
Here the layerwise optimization norm is $\|\cdot\|_\infty$, and the dual is $\|\cdot\|_1$.  
For a matrix with $d_\ell = m_\ell n_\ell$ entries,
\[
\sup_{\|A\|_\infty = 1} \|A\|_F^2 = d_\ell,
\qquad
\|A\|_1 \le d_\ell \|A\|_\infty.
\]
Accumulating across layers yields
\[
\kappa_{\mathrm{Lion}} = \sum_{\ell=1}^{L} m_\ell n_\ell.
\]

\item \textbf{Normalized SGD (Frobenius norm).}  
In this case, the per-layer norm is already the Frobenius norm, and the global norm is
\[
\|W\| = \max_\ell \|W_\ell\|_F.
\]
The corresponding dual norm aggregates Frobenius contributions over layers, giving
\[
\sup_{\max_\ell \|U_\ell\|_F = 1} \sum_{\ell=1}^L \|U_\ell\|_F^2 = L.
\]
Thus,
\[
\kappa_{\mathrm{normSGD}} = L.
\]
\end{itemize}

Intuitively, $\kappa$ measures the worst-case amplification of the effective step when mapping from the LMO geometry to Frobenius scale. The differences above reflect the fundamental geometry of the optimizers: spectral structure for Muon, elementwise bounds for Lion, and layerwise normalization for normalized SGD. These forms of~$\kappa$ are used in all our experiments (see Section~\ref{sec:exp}).

\subsection{Closed-form solution for \texorpdfstring{$K_0,K_1,K_2$}{K0, K1, K2} and selection of \texorpdfstring{$\Delta'$}{Delta prime}}
\label{app:K_012}

We work with the practical schedule
\[
\eta(\Delta)
=
\frac{\Delta}{K_0 + K_1 \Delta + K_2 \Delta^2},
\]
with three unknown coefficients $K_0,K_1,K_2$ and three constraints:

\begin{enumerate}
\item[(i)] $\eta'(\Delta')=0$ (unique critical point at $\Delta'$),
\item[(ii)] $\eta(\Delta') = \mathrm{lr}$ (peak learning rate),
\item[(iii)] $\eta(\Delta^0) = \mathrm{lr}/\mathrm{div}$ (initial warmup value).
\end{enumerate}

\paragraph{Step 1: critical point.}
Differentiating $\eta(\Delta)$,
\[
\eta'(\Delta)
=
\frac{(K_0 + K_1 \Delta + K_2 \Delta^2) - \Delta(K_1 + 2K_2\Delta)}
     {(K_0 + K_1 \Delta + K_2 \Delta^2)^2}
=
\frac{K_0 - K_2 \Delta^2}{(K_0 + K_1 \Delta + K_2 \Delta^2)^2}.
\]
Setting $\eta'(\Delta')=0$ yields
\begin{equation}
K_0 = K_2 (\Delta')^2.
\label{eq:K0}
\end{equation}

\paragraph{Step 2: value constraints.}
Condition (ii) gives
\[
\frac{\Delta'}{K_0 + K_1 \Delta' + K_2 (\Delta')^2} = \mathrm{lr}.
\]
Using \eqref{eq:K0},
\[
\frac{\Delta'}{2K_2(\Delta')^2 + K_1\Delta'}
=
\mathrm{lr}
\quad\Longrightarrow\quad
1 = \mathrm{lr}(2K_2\Delta' + K_1).
\]
Solving for $K_1$:
\begin{equation}
K_1
=
\frac{1}{\mathrm{lr}} - 2K_2\Delta'.
\label{eq:K1}
\end{equation}

Condition (iii) gives
\[
\frac{\Delta^0}{K_0 + K_1\Delta^0 + K_2(\Delta^0)^2}
=
\frac{\mathrm{lr}}{\mathrm{div}}.
\]
Substituting \eqref{eq:K0},
\begin{equation}
\frac{\Delta^0}{K_2(\Delta')^2 + K_1\Delta^0 + K_2(\Delta^0)^2}
=
\frac{\mathrm{lr}}{\mathrm{div}}.
\label{eq:init_constraint}
\end{equation}

\paragraph{Step 3: solving for $K_2,K_1,K_0$.}
Insert \eqref{eq:K1} into \eqref{eq:init_constraint} and rearrange to obtain
\[
(\mathrm{div}-1)\Delta^0
=
\mathrm{lr} K_2(\Delta^0-\Delta')^2.
\]
Thus
\begin{equation}
K_2
=
\frac{\Delta^0(\mathrm{div}-1)}
     {\mathrm{lr}(\Delta^0-\Delta')^2}.
\label{eq:K2}
\end{equation}

Then
\begin{equation}
K_0
=
K_2(\Delta')^2
=
\frac{\Delta^0(\Delta')^2(\mathrm{div}-1)}
     {\mathrm{lr}(\Delta^0-\Delta')^2}.
\label{eq:K0_final}
\end{equation}

Finally use \eqref{eq:K1} and \eqref{eq:K2} to obtain
\begin{equation}
K_1
=
\frac{(\Delta^0)^2 - 2\Delta^0\Delta' \mathrm{div} + (\Delta')^2}
     {\mathrm{lr}(\Delta^0-\Delta')^2}.
\label{eq:K1_final}
\end{equation}

It is immediate to verify that \eqref{eq:K0}, \eqref{eq:K1_final}, \eqref{eq:K2} jointly satisfy (i)--(iii).

\paragraph{Step 4: selecting $\Delta'$.}
The coefficients above hold for any $\Delta'$ with $\Delta'\neq\Delta^0$, leaving $\Delta'$ as a free scalar.  
We choose it by minimizing the MSE between the schedule and a target warmup+decay curve:
\[
\Delta'
    =
    \arg\min_{\Delta > 0} \left\{
    \mathbb{E}_{\tilde{\Delta} \sim \mathcal{N}(\Delta; \sigma_F^2 / \kappa)}
        \bigl[(\eta(\tilde{\Delta}) - \eta_{\mathrm{trgt}}(\tilde{\Delta}))^2\bigr]
    \right\},
\]
In practice, we evaluate $1000$ candidate values of $\Delta'$, form $(K_0,K_1,K_2)$ using \eqref{eq:K2}--\eqref{eq:K1_final}, enforce constraints numerically, and choose the minimizer. This is done once at initialization and incurs negligible runtime overhead.

\section{Illustration of the Adaptive Scheduler}
\label{sec:scheduler_illustration}
For illustration, Figure~\ref{fig:scheduler-lion} shows the resulting learning rate profile of Algorithm~\ref{alg:lr_scheduler} instantiated with the Lion optimizer.
\begin{figure}[H]
    \centering
    \includegraphics[width=0.5\columnwidth]{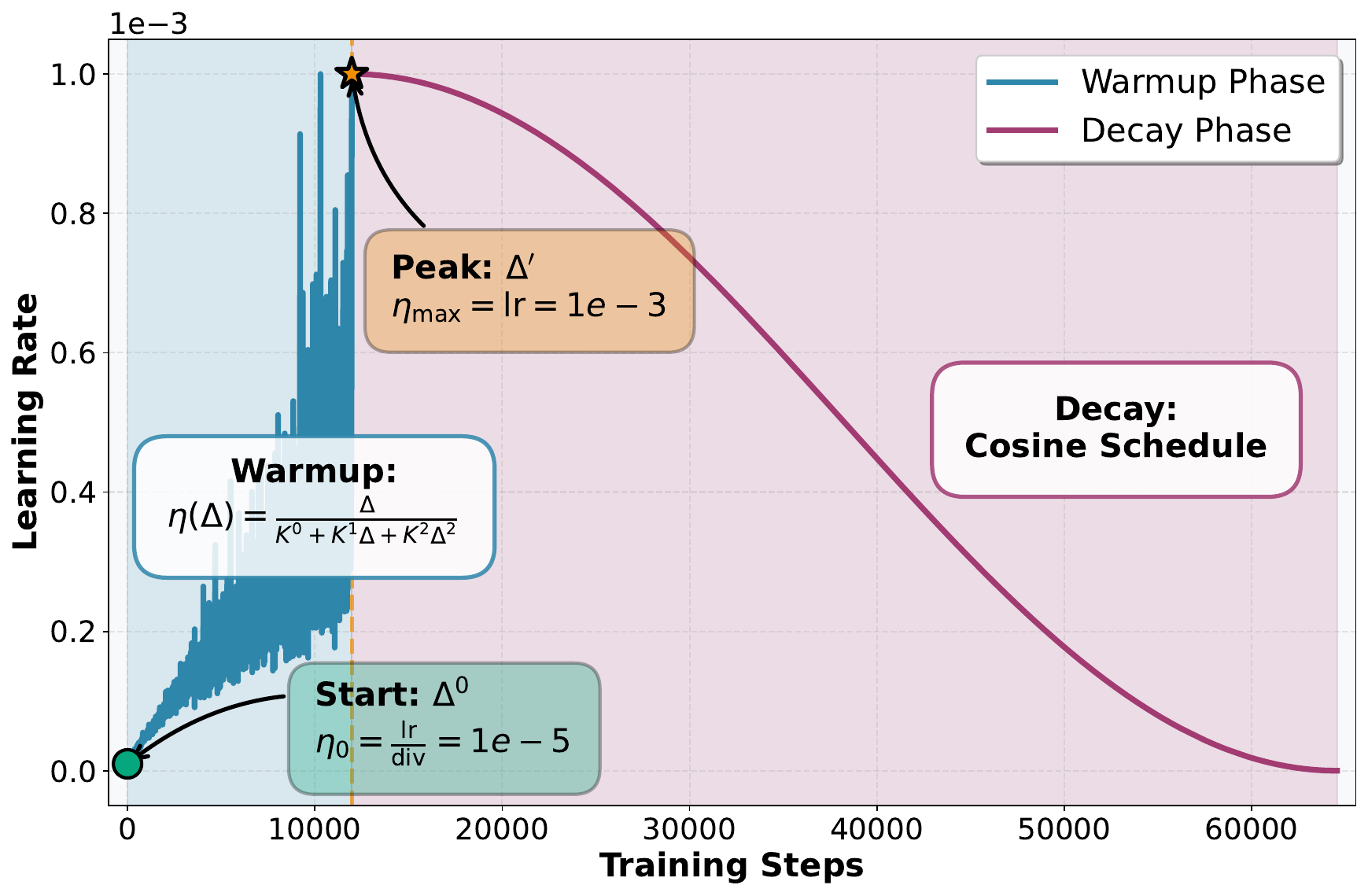}
    \caption{Example of learning rate schedule produced by Algorithm~\ref{alg:lr_scheduler} for the Lion optimizer on FineWeb. }
    \label{fig:scheduler-lion}
\end{figure}

\section{Connection \texorpdfstring{$(\rho, K_0, K_1, K_\rho)$}{(rho, K\_0, K\_1, K\_rho)}-smoothness with \texorpdfstring{$(\rho, L_0, L_\rho)$}{(rho, L\_0, L\_rho)}-smoothness}
\label{sec:connection_with_smoothness}

The following Lemma~\ref{lem:smoothness_equivalence} connects the standard $(\rho, L_0, L_\rho)$-smoothness (Assumption~\ref{ass:smoothness1}) with the $\left(\rho, K_0, K_1, K_\rho\right)$-smoothness (Assumption~\ref{ass:smoothness2}).

\begin{assumption}
\label{ass:smoothness1}
The function $f : \mathcal{X} \to \mathbb{R}$ is $(\rho, L_0, L_\rho)$-smooth, i.e., there exist $L_0, L_\rho \ge 0$, $\rho>0$ such that
\[
\|\nabla f(x) - \nabla f(y)\|_{\star} \le (L_0 + L_\rho \|\nabla f(x)\|^{\rho}_{\star}) \|x - y\|
\]
for all $x, y \in \mathcal{X}$.
\end{assumption}

\begin{lemma}
\label{lem:smoothness_equivalence}
If the function $f$ is $(\rho, L_0, L_\rho)$-smooth with $0<\rho<2$, then it is $\left(\frac{\rho}{2-\rho}, K_0, K_1, K_\rho\right)$-smooth.
\end{lemma}

\begin{proof}
It is easy to show that if $f$ is $(\rho, L_0, L_\rho)$-smooth, then for all $x, y \in \mathcal{X}$,
\[
f(y) \le f(x) + \langle \nabla f(x), y - x \rangle + \frac{L_0 + L_\rho \|\nabla f(x)\|^{\rho}_{\star}}{2} \|y - x\|^2.
\]
By taking $y = x + \frac{\|\nabla f(x)\|_{\star}}{L_0 + L_\rho \|\nabla f(x)\|^{\rho}_{\star}} \operatorname{LMO}(\nabla f(x))$, we have
$f^{\star} \le f(y) \le f(x) - \frac{\|\nabla f(x)\|_{\star}^2}{2L_0 + 2L_\rho \|\nabla f(x)\|^{\rho}_{\star}}$.
Denoting $g := \|\nabla f(x)\|_{\star}$ and $\Delta := f(x) - f^{\star}$, this gives $\Delta \ge \frac{g^2}{2(L_0 + L_\rho g^{\rho})}$.

We now upper bound $g^{\rho}$ in terms of $\Delta$. If $g^{\rho} \le \frac{L_0}{L_\rho}$, then trivially $g^{\rho} \le \frac{L_0}{L_\rho}$. Otherwise, $L_0 + L_\rho g^{\rho} \le 2L_\rho g^{\rho}$, hence $\Delta \ge \frac{g^{2-\rho}}{4L_\rho}$. Since $0 < \rho < 2$, raising both sides to the power $\frac{\rho}{2-\rho}$ gives $g^{\rho} \le (4L_\rho \Delta)^{\frac{\rho}{2-\rho}}$. Combining both cases via $\max\{a, b\} \le a + b$:
\[
g^{\rho} \le \frac{L_0}{L_\rho} + (4L_\rho \Delta)^{\frac{\rho}{2-\rho}}, \quad \text{so} \quad L_0 + L_\rho g^{\rho} \le 2L_0 + L_\rho(4L_\rho)^{\frac{\rho}{2-\rho}} \Delta^{\frac{\rho}{2-\rho}}.
\]
Plugging this into Assumption~\ref{ass:smoothness1} yields Assumption~\ref{ass:smoothness2} with exponent $\frac{\rho}{2-\rho}$, constants $K_0 := 2L_0$, $K_1 := 0$, and $K_\rho := L_\rho(4L_\rho)^{\frac{\rho}{2-\rho}}$.
\end{proof}

\section{Proofs}
\label{app:proofs}

\begin{corollary}
\label{cor:star_convex_gradient}
If $f$ is differentiable and star-convex (Assumption~\ref{ass:star_convexity}), then for all $x \in \mathcal{X}$:
\[
f(x) - f^\star \le \langle \nabla f(x), x - x^\star \rangle.
\]
\end{corollary}
\begin{proof}
From star-convexity, $f(x + \beta(x^\star - x)) - f(x) \le -\beta (f(x) - f^\star)$ for $\beta \in [0,1]$. Dividing by $\beta$ and taking $\beta \to 0^+$ yields the directional derivative $\langle \nabla f(x), x^\star - x \rangle \le -(f(x) - f^\star)$.
\end{proof}

\subsection{Proof of Theorem~\ref{thm:1}}
\begin{proof}
Denote $\mathcal{K}^t := K_0 + K_1 \Delta^t + K_\rho(\Delta^t)^\rho$. By Assumption \ref{ass:smoothness2},
\[
\Delta^{t+1} \le \Delta^t + \langle \nabla f(x^t), x^{t+1} - x^t \rangle + \frac{\mathcal{K}^t}{2} \|x^{t+1} - x^t\|^2.
\]
Since 
\[
\langle \nabla f(x^t), x^{t+1} - x^t \rangle = \eta^t \langle \nabla f(x^t), \operatorname{LMO}(\nabla f(x^t)) \rangle = - \eta^t \|\nabla f(x^t)\|_{\star},
\]
then
\[
\Delta^{t+1} \le \Delta^t - \eta^t \|\nabla f(x^t)\|_{\star} + \frac{\mathcal{K}^t}{2} (\eta^t)^2.
\]
By Corollary~\ref{cor:star_convex_gradient} and the boundedness of the iterates $x^t$, we have
$\Delta^t = f(x^t) - f^{\star} \le \langle \nabla f(x^t), x^t - x^{\star} \rangle \le \|\nabla f(x^t)\|_{\star} \|x^t - x^{\star} \| \le D  \|\nabla f(x^t)\|_{\star}$,
where $D>0$. Therefore,
\[
\Delta^{t+1} \le \Delta^t - \frac{\Delta^t}{D} \eta^t + \frac{\mathcal{K}^t}{2} (\eta^t)^2.
\]
The right hand side is minimized by setting $\eta^t = \frac{\Delta^{t}}{D \mathcal{K}^t}$, thus $\Delta^{t+1} \le \Delta^t$ and
\[
\Delta^{t+1} \le \Delta^{t}  - \frac{(\Delta^t)^2}{2D^2 \mathcal{K}^t} \le \Delta^{t}  - \frac{\Delta^t \Delta^{t+1}}{2D^2 \mathcal{K}^t}.
\]
Telescoping the sum after dividing both sides by $\Delta^t \Delta^{t+1}$ yields
$\sum_{t=0}^{T-1} \frac{1}{2D^2 \mathcal{K}^t} \le \frac{1}{\Delta^{T}} - \frac{1}{\Delta^0}$.
By the Cauchy-Schwarz inequality,
$\sum_{t=0}^{T-1} \frac{1}{2D^2 \mathcal{K}^t} \ge \frac{T^2}{\sum_{t=0}^{T-1}2D^2 \mathcal{K}^t}$,
which gives
$\frac{T^2}{\sum_{t=0}^{T-1}2D^2 \mathcal{K}^t} \le \frac{1}{\Delta^{T}} - \frac{1}{\Delta^0} \le \frac{1}{\Delta^{T}}$.
Therefore,
\[
\Delta^T \le \frac{2D^2\sum_{t=0}^{T-1}\mathcal{K}^t}{T^2}.
\]
It is easy to show for $\rho>1$ that since $\Delta^t$ is non-increasing, then the learning rate $\eta^t = \frac{\Delta^{t}}{D \mathcal{K}^t}$ is non-decreasing for $\Delta^t \ge \left(\frac{K_0}{(\rho-1)K_\rho}\right)^{\frac{1}{\rho}}$ (warm-up stage) and non-increasing thereafter (decay stage).
\end{proof}

\subsection{Proof of Theorem~\ref{thm:2}}
\begin{proof}
Denote $\mathcal{K}^t := K_0 + K_1 \Delta^t + K_\rho(\Delta^t)^\rho$. By Assumption \ref{ass:smoothness2},
\[
\Delta^{t+1} \le \Delta^t + \langle \nabla f(x^t), x^{t+1} - x^t \rangle + \frac{\mathcal{K}^t}{2} \|x^{t+1} - x^t\|^2.
\]

Note that $g(u)=\frac{u}{8(K_0 + K_1 u + K_\rho u^{\rho})}$, $u>0$, $\rho>1$, is maximized at $u=\left(\frac{K_0}{K_\rho(\rho - 1)}\right)^{\frac{1}{\rho}}$ with a maximum value of $\left[ 8 \left( \rho \left( \frac{K_0}{\rho - 1} \right)^{\frac{\rho - 1}{\rho}} K_{\rho}^{\frac{1}{\rho}} + K_1 \right) \right]^{-1}$, so $0<\lambda \eta^t \leq 1$ for $0< \lambda \le \frac{1}{\max\left(\|x^0\|, \|x^{\star}\|, 1/\lambda_{\max}\right)}$, where $\lambda_{\max}=\left[ 8 \left( \rho \left( \frac{K_0}{\rho - 1} \right)^{\frac{\rho - 1}{\rho}} K_{\rho}^{\frac{1}{\rho}} + K_1 \right) \right]^{1/2}$.



Let us prove $\|x^t\| \le 1/\lambda$ for all $t \ge 0$ by induction.
The induction base ($t=0$) holds trivially.
Assuming $\|x^t\| \le 1/\lambda$ holds, the update rule and the bound $\|\text{LMO}(g^t)\| \le 1$ yield:
\begin{equation*}
    \|x^{t+1}\| = \| (1-\lambda \eta^t)x^t + \eta^t \operatorname{LMO}(g^t) \| \le (1 - \lambda \eta^t)\|x^t\| + \eta^t \le (1 - \lambda \eta^t)\frac{1}{\lambda} + \eta^t = \frac{1}{\lambda},
\end{equation*}
where $1 - \lambda \eta^t \ge 0$, since $0<\lambda \eta^t \leq 1$. Thus, $\lambda \|x^t\| \le 1$ for all $t \ge 0$.

Using these bounds, we obtain
\begin{equation*}
    \|x^{t+1} - x^t\| = \eta^t \|\text{LMO}(g^t) - \lambda x^t\| \le \eta^t (1 + \lambda \|x^t\|) \le 2\eta^t,
\end{equation*}
and
\begin{equation*}
    \|x^t - x\| = \lambda \eta^t \|x^t - x^{\star}\| \le \eta^t (\lambda \|x^t\| + \lambda \|x^{\star}\|) \le 2\eta^t,
\end{equation*}
where $x := (1 - \lambda \eta^t)x^t + \lambda \eta^t x^{\star}$.

Therefore,
\[
\Delta^{t+1} \le \Delta^t + \langle \nabla f(x^t), x^{t+1} - x^t \rangle + 2\mathcal{K}^t (\eta^t)^2.
\]

Let us construct an upper bound for $\langle \nabla f(x^t), x^{t+1} - x^t \rangle$:
\begin{align*}
\langle \nabla f(x^t), x^{t+1} - x^t \rangle &= \langle \nabla f(x^t), x - x^t \rangle + \eta^t \langle \nabla f(x^t), \operatorname{LMO}(\nabla f(x^t)) - \lambda x^{\star} \rangle \\
&\le \langle \nabla f(x^t), x - x^t \rangle - \eta^t \| \nabla f(x^t) \|_{\star} + \eta^t \lambda \|x^{\star}\| \| \nabla f(x^t) \|_{\star} \\
& \le \langle \nabla f(x^t), x - x^t \rangle.
\end{align*}

By Assumption \ref{ass:smoothness2}, 
\[
\langle \nabla f(x^t), x - x^t \rangle \le f(x) - f(x^t) + \frac{\mathcal{K}^t}{2} \|x^t - x\|^2 \le f(x) - f(x^t) + 2\mathcal{K}^t (\eta^t)^2.
\]

By Assumption \ref{ass:star_convexity}, 
\[
f(x) - f(x^t) = f((1-\lambda \eta^t)x^t+\lambda \eta^t x^{\star}) - f(x^t) \le -\lambda \eta^t \Delta^t.
\]

Combining these:
\[
\langle \nabla f(x^t), x^{t+1} - x^t \rangle \le -\lambda \eta^t \Delta^t  + 2 \mathcal{K}^t (\eta^t)^2,
\]
and thus 
\[
\Delta^{t+1} \le \Delta^t -\lambda \eta^t \Delta^t + 4 \mathcal{K}^t (\eta^t)^2.
\]

The right hand side is minimized by setting $\eta^t = \frac{\lambda\Delta^{t}}{8\mathcal{K}^t}$, thus $\Delta^{t+1} \le \Delta^t$ and
\[
\Delta^{t+1} \le \Delta^{t}  - \frac{\lambda^2(\Delta^t)^2}{16\mathcal{K}^t} \le \Delta^{t} - \frac{\lambda^2\Delta^t \Delta^{t+1}}{16\mathcal{K}^t}.
\]
The rest of the proof is straightforward.
\end{proof}

\subsection{Proof of Theorem~\ref{thm:3}}
\begin{proof}
Denote $\mathcal{K}^t_\xi := K_0 + K_1 \Delta^t_{\xi} + K_\rho(\Delta^t_{\xi})^{\rho}$ and $\bar{\mathcal{K}} := K_0 + K_1 M + K_\rho M^{\rho}$. By the update rule,
\[
\|x^{t+1}-x^{\star}\|^2 
= \|x^{t}-x^{\star}\|^2  - \frac{2\eta^t}{\|\nabla f_{\xi^t} (x^t)\|} \langle \nabla f_{\xi^t} (x^t), x^t - x^{\star} \rangle + (\eta^t)^2
\]
Using Corollary \ref{cor:star_convex_gradient} and Assumptions \ref{ass:smoothness2_weak}, \ref{ass:overparam}, we have
\[
\|x^{t+1}-x^{\star}\|^2 \leq \|x^{t}-x^{\star}\|^2  - \frac{2\eta^t \Delta_{\xi}^t}{\|\nabla f_{\xi^t} (x^t)\|}  + (\eta^t)^2 \leq \|x^{t}-x^{\star}\|^2  - \frac{2\eta^t \Delta_{\xi}^t}{\mathcal{K}^t_\xi \|x^t-x^{\star}\|}  + (\eta^t)^2.
\] 

One can verify that $\|x^1 - x^{\star}\| \leq \|x^0 - x^{\star}\| =: D$ for $t=0$ almost surely. By induction, $\|x^{t} - x^{\star}\| \leq D$ for all $t \geq 0$ almost surely. Consequently,
\[
\|x^{t+1}-x^{\star}\|^2 
\leq \|x^{t}-x^{\star}\|^2 - \frac{2\eta^t \Delta_{\xi}^t}{D \cdot \mathcal{K}^t_\xi} + (\eta^t)^2 
= \|x^{t}-x^{\star}\|^2 - \frac{(\Delta_{\xi}^t)^2}{D^2 (\mathcal{K}^t_\xi)^2}.
\]
After taking expectation and summing from $t=0$ to $T-1$, we obtain
\[
\mathbb{E}\bigl[\|x^{T}-x^{\star}\|^2\bigr]
\le D^2 - \frac{1}{D^2}\sum_{t=0}^{T-1} \mathbb{E}\left[\frac{(\Delta_{\xi}^t)^2}{(\mathcal{K}^t_\xi)^2}\right].
\]
Since $\mathbb{E}[\|x^{T}-x^{\star}\|^2] \ge 0$, we have $\sum_{t=0}^{T-1} \mathbb{E}[(\Delta_{\xi}^t)^2 / (\mathcal{K}^t_\xi)^2] \le D^4$.

We use the inequality: for real $X$ and positive $Y$ with finite expectations, $\mathbb{E}[X^2/Y] \ge (\mathbb{E}[X])^2 / \mathbb{E}[Y]$. By applying this with $X = \Delta_{\xi}^t$ and $Y = (\mathcal{K}^t_\xi)^2$, we have
\[
\mathbb{E}\left[\frac{(\Delta_{\xi}^t)^2}{(\mathcal{K}^t_\xi)^2}\right]
\ge \frac{(\mathbb{E}[\Delta_{\xi}^t])^2}{\mathbb{E}[(\mathcal{K}^t_\xi)^2]}.
\]
Combining with the Cauchy-Schwarz inequality gives
\[
\frac{\bigl(\sum_{t=0}^{T-1}\mathbb{E}[\Delta_{\xi}^t]\bigr)^2}{\mathbb{E}\bigl[\sum_{t=0}^{T-1} (\mathcal{K}^t_\xi)^2\bigr]} \le D^4.
\]
Dividing by $T^2$ and taking square root yields the general bound:
\begin{equation}
\frac{1}{T}\sum_{t=0}^{T-1}\mathbb{E}[\Delta_{\xi}^t] \le \frac{D^2}{T}\sqrt{\mathbb{E}\Bigl[\sum_{t=0}^{T-1} (\mathcal{K}^t_\xi)^2\Bigr]}.
\end{equation}

If $\Delta_{\xi}^t \le M$ a.s.\ for all $t$, then $(\mathcal{K}^t_\xi)^2 \le \bar{\mathcal{K}}^2$, and thus
\[
\frac{1}{T}\sum_{t=0}^{T-1}\mathbb{E}[\Delta_{\xi}^t] \le \frac{D^2 \bar{\mathcal{K}}}{\sqrt{T}} = \mathcal{O}(T^{-1/2}).
\]
\end{proof}
\newpage
\section{Experimental Details}
\label{app:experimental_details}

\subsection{Model Architecture}

All experiments were conducted on NVIDIA H200 GPUs with 140GB memory. Table~\ref{tab:model_arch} describes the Llama-based architecture used in our experiments. With $n_{\text{layer}} = 12$, this configuration yields a 124M parameter model. Increasing $n_{\text{layer}}$ to 24 produces a 210M parameter model.

\begin{table}[h!]
\centering
\caption{Model architecture hyperparameters. The configuration yields a 124M parameter Llama model.}
\label{tab:model_arch}
\begin{tabular}{lll}
\toprule
\textbf{Parameter} & \textbf{Value} & \textbf{Description} \\
\midrule
\multirow{ 2}{*}{$n_{\text{layer}}$} & 12 (for 124M model) & \multirow{ 2}{*}{Number of transformer layers} \\
& 24 (for 210M model) & \\
$n_{\text{embd}}$ & 768 & Embedding dimension \\
$n_{\text{head}}$ & 12 & Number of attention heads \\
Vocabulary size & 50257 & Size of tokenizer vocabulary \\
\bottomrule
\end{tabular}
\end{table}

\subsection{Assumption Validation Experiments}
\label{appendix:hyper_ass}
Table~\ref{tab:assumption_hyperparams} lists the training hyperparameters used for validating the generalized smoothness assumption (Section~\ref{sec:ass}). The parameters are identical across all three optimizers (signSGD, Muon, normalized SGD) except for the learning rate and the number of training iterations, which are shown in separate columns.

\begin{table}[h!]
\centering
\caption{Training hyperparameters for assumption validation experiments on FineWeb dataset for Figures \ref{fig:smoothness-lion} from Section \ref{sec:ass} and Figure \ref{fig:lips-muon-norm} from Appendix \ref{app:smooth}.}
\label{tab:assumption_hyperparams}
\begin{tabular}{lccc}
\toprule
\textbf{Parameter} & \textbf{signSGD} & \textbf{Muon} & \textbf{normSGD} \\
\midrule
Learning rate ($\mathrm{lr}$) & $10^{-4}$ & $10^{-4}$ & $10^{-3}$ \\
Training iterations & 10k & 15k & 10k \\
\midrule
Batch size & \multicolumn{3}{c}{64} \\
Sequence length & \multicolumn{3}{c}{512} \\
Gradient accumulation steps & \multicolumn{3}{c}{1} \\
Warmup steps & \multicolumn{3}{c}{1000} \\
Decay scheduler & \multicolumn{3}{c}{Cosine} \\
Gradient clipping & \multicolumn{3}{c}{0.5} \\
Weight decay & \multicolumn{3}{c}{0.1} \\
Momentum & \multicolumn{3}{c}{0.9} \\
Divisor ($\mathrm{div}$) & \multicolumn{3}{c}{100} \\
Target loss $f^\star$ & \multicolumn{3}{c}{3.2} \\
Smoothness exponent $\rho$ & \multicolumn{3}{c}{2} \\
\bottomrule
\end{tabular}
\end{table}

The common hyperparameters follow the setup from \cite{semenov2025benchmarking}. The parameters are identical across all three optimizers (signSGD, Muon, normalized SGD) except for the learning rate and the number of training iterations.
The learning rates are chosen to ensure comparable training dynamics across optimizers: signSGD use $\mathrm{lr} = 10^{-4}$, while normalized SGD requires $\mathrm{lr} = 10^{-3}$ due to its smaller effective step size in Frobenius norm (see Section \ref{sec:our_sceduler} for a detailed discussion of geometry-dependent scaling). Muon uses $\mathrm{lr} = 10^{-4}$ to carefully capture the smoothness landscape. The number of training iterations is adjusted to capture sufficient data for fitting the smoothness ratio: Muon requires 15k iterations to reach lower loss values (since it uses quite low $\mathrm{lr}$), while signSGD and normalized SGD converge faster and use 10k iterations.

\subsection{Main Experiments Hyperparameters}
\label{appendix:hyper}
\begin{table}[h!]
\centering
\caption{Training hyperparameters for main experiments on FineWeb dataset for Figures \ref{fig:warmup-sweep-124M} and Figures \ref{fig:warmup-sweep-210M} from Section \ref{sec:exp}.}
\label{tab:main_hyperparams}
\resizebox{\linewidth}{!}{
\begin{tabular}{llcccc}
\toprule
\textbf{Optimizer} & \textbf{Parameter} & \textbf{124M, BS=256} & \textbf{210M, BS=256} & \textbf{124M, BS=32} & \textbf{210M, BS=32} \\
\midrule
\multirow{4}{*}{Muon} 
& Learning rate & $2 \times 10^{-3}$ & $3 \times 10^{-3}$ & $10^{-3}$ & $10^{-3}$ \\
& Target loss $f^\star$ & \multicolumn{4}{c}{3.3 (1B train tokens), 3.2 (2.1 train tokens)} \\
& Momentum ($\beta_1$) & \multicolumn{4}{c}{0.8} \\
& Divisor ($\mathrm{div}$) & \multicolumn{4}{c}{100} \\
\midrule
\multirow{5}{*}{Lion} 
& Learning rate & \multicolumn{4}{c}{$10^{-3}$} \\
& Target loss $f^\star$ & \multicolumn{4}{c}{3.4 (1B train tokens), 3.3 (2.1 train tokens)} \\
& Momentum ($\beta_1$) & \multicolumn{4}{c}{0.9} \\
& Second momentum ($\beta_2$) & \multicolumn{4}{c}{0.99} \\
& Divisor ($\mathrm{div}$) & \multicolumn{4}{c}{100} \\
\midrule
\multirow{4}{*}{normSGD} 
& Learning rate & \multicolumn{4}{c}{$\{ 10^{-4}, 10^{-3}, \mathbf{10^{-2}}, 10^{-1}\}$} \\
& Target loss $f^\star$ & \multicolumn{4}{c}{4.3 (1B train tokens), 4.1 (2.1 train tokens)} \\
& Momentum & \multicolumn{4}{c}{$\{ 0.8, 0.9, \textbf{0.95}, 0.99\}$} \\
& Divisor ($\mathrm{div}$) & \multicolumn{2}{c}{$\{ 2, \textbf{5}, 10, 100 \}$} & \multicolumn{2}{c}{$\{ \textbf{2}, 5, 10, 100 \}$} \\
\midrule
\multicolumn{6}{l}{\textbf{Common hyperparameters (all optimizers)}} \\
\multicolumn{2}{l}{Sequence length} & \multicolumn{4}{c}{512} \\
\multicolumn{2}{l}{Gradient accumulation} & \multicolumn{4}{c}{1} \\
\multicolumn{2}{l}{Decay scheduler} & \multicolumn{4}{c}{Cosine} \\
\multicolumn{2}{l}{Gradient clipping} & \multicolumn{4}{c}{0.5} \\
\multicolumn{2}{l}{Weight decay} & \multicolumn{4}{c}{0.1} \\
\multicolumn{2}{l}{Dropout} & \multicolumn{4}{c}{0.0} \\
\multicolumn{2}{l}{Frobenius variance $\sigma_F^2$} & \multicolumn{4}{c}{$10^{3}$} \\
\bottomrule
\end{tabular}
}
\end{table}

For Muon and Lion optimizers, we use the hyperparameters from \cite{semenov2025benchmarking} without additional tuning. For normalized SGD (normSGD), we perform a hyperparameter sweep over learning rate, momentum, and divisor. The best configurations are highlighted in bold. Interestingly, we observe that the optimal divisor depends on the batch size: $\mathrm{div} = 5$ performs best for $\mathrm{bs} = 256$, while $\mathrm{div} = 2$ is optimal for $\mathrm{bs} = 32$.

\newpage

\subsection{Vision Transformer Experiment Hyperparameters}
\label{appendix:hyper_swin}

\begin{table}[h!]
\centering
\caption{Training hyperparameters for the Swin-B ImageNet-1K experiment (Section~\ref{sec:llm_results}, Figure~\ref{fig:swin}). Bold indicates the selected value.}
\label{tab:swin_hyperparams}
\begin{tabular}{lc}
\toprule
\textbf{Parameter} & \textbf{Value} \\
\midrule
Model & Swin-B (swin\_base\_patch4\_window7\_224) \\
Dataset & ImageNet-1K \\
Optimizer & Muon \\
Learning rate & $\{5 \times 10^{-4},\ \mathbf{10^{-3}},\ 5 \times 10^{-3}\}$ \\
Batch size & 128 \\
Iterations & 50{,}000 \\
Weight decay & 0.1 \\
Decay scheduler & Cosine \\
\midrule
\multicolumn{2}{l}{Adaptive scheduler (ours)} \\
Frobenius variance $\sigma_F^2$ & $10^{-3}$ \\
Smoothness exponent $\rho$ & 2 \\
Target loss $f^\star$ & 0.4 \\
\bottomrule
\end{tabular}
\end{table}

\end{appendixpart}
\end{document}